\documentclass[a4paper, 10 pt, conference]{ieeeconf}  
\usepackage[colorlinks,urlcolor=blue,linkcolor=black,citecolor=black]{hyperref}
\IEEEoverridecommandlockouts                              

\usepackage{graphicx}
\usepackage[update]{epstopdf}
\usepackage[process=all,crop=pdfcrop,crossref=false,cleanup={.dvi,.ps,.pdf,.log,.aux,.bbl}]{pstool}
\usepackage{relsize}
\usepackage{amsmath}
\usepackage{amssymb} 
\usepackage{booktabs}
\usepackage{multirow}
\usepackage{fancyvrb} 
\usepackage{changepage}
\usepackage{siunitx}
\usepackage[skip=5pt]{caption}
\usepackage{import}
\usepackage{ragged2e}
\usepackage{cite}
\usepackage{amsfonts}

\newtheorem{definition}{Definition}

\newtheorem{theorem}{Theorem}

\newtheorem{remark}{Remark}

\newtheorem{example}{Example}

\newcommand{\R}{\mathbb{R}}
\newcommand{\Rn}{\mathbb{R}^n}

\newcommand{\polyCons}{s} 		
\newcommand{\numObs}{o}			
\newcommand{\numInt}{z}			
\newcommand{\numDist}{z}		
\newcommand{\pieces}{N}			
\newcommand{\ind}{r}

\newcommand{\subRL}{a}
\newcommand{\captionmath}[1]{{\fontfamily{rmdefault}\selectfont \boldmath$#1$}}

\DeclareMathSymbol{\shortminus}{\mathbin}{AMSa}{"39}
\newcommand\mydots{\hbox to 1em{.\hss.\hss.\hss}}

\newcommand\eatpunct[1]{}

\title{Provably Safe Reinforcement Learning via Action Projection using Reachability Analysis and Polynomial Zonotopes} 

\author{Niklas Kochdumper$^{*,1,2}$, Hanna Krasowski$^{*,1}$, Xiao Wang$^{*,1}$, Stanley Bak$^2$, and Matthias Althoff$^1$
\thanks{$^*$The first three authors contributed equally.}
\thanks{$^1$Department of Computer Science, Technical University of Munich, 85748 Garching, Germany}
\thanks{$^2$Department of Computer Science, Stony Brook University, Stony Brook, NY 11794 USA}
}

\begin{document}
	
\maketitle
\thispagestyle{empty}
\pagestyle{empty}


\begin{abstract}
	While reinforcement learning produces very promising results for many applications, its main disadvantage is the lack of safety guarantees, which prevents its use in safety-critical systems. In this work, we address this issue by a safety shield for nonlinear continuous systems that solve reach-avoid tasks. Our safety shield prevents applying potentially unsafe actions from a reinforcement learning agent by projecting the proposed action to the closest safe action. This approach is called action projection and is implemented via mixed-integer optimization. The safety constraints for action projection are obtained by applying parameterized reachability analysis using polynomial zonotopes, which enables to accurately capture the nonlinear effects of the actions on the system. In contrast to other state-of-the-art approaches for action projection, our safety shield can efficiently handle input constraints and dynamic obstacles, eases incorporation of the spatial robot dimensions into the safety constraints, guarantees robust safety despite process noise and measurement errors, and is well suited for high-dimensional systems, as we demonstrate on several challenging benchmark systems.
\end{abstract}



\section{Introduction} 
\label{sec:intro}

Reinforcement learning has been successfully applied to find solutions for many challenging applications, such as robotics \cite{Zhao2020}, autonomous driving \cite{Kiran2021}, and power systems \cite{Zhang2019}. Many of these applications are safety-critical, so that the lack of safety guarantees for standard reinforcement learning controllers prevents their deployment in the real world. We aim to overcome this limitation with a novel safety shield for reinforcement learning agents that considers the very general case of disturbed nonlinear continuous systems with input constraints that have to avoid dynamic obstacles. Note that our safety shield can be applied to arbitrary unsafe controllers, while reinforcement learning is the main focus of this work.

\begin{figure}[!tb]
	\centering                                                                                        
	\setlength{\belowcaptionskip}{-15pt}
	\psfragfig[width=0.98\columnwidth]{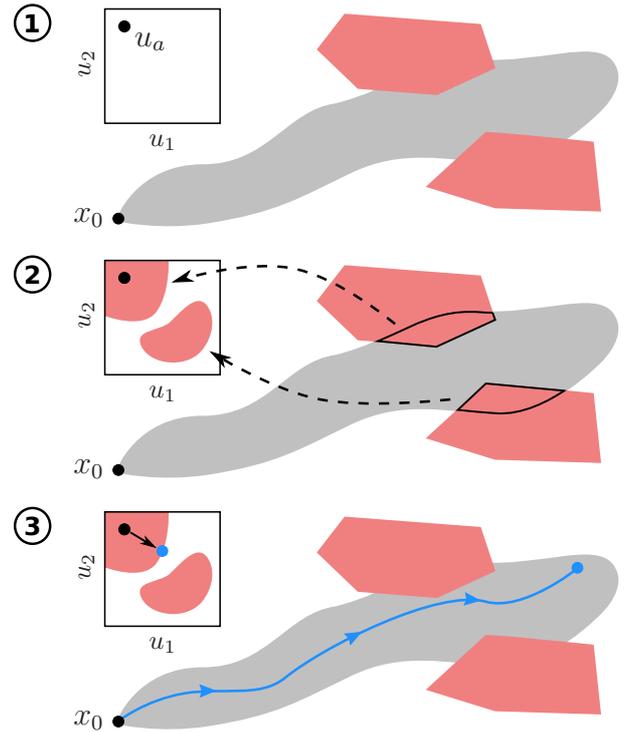}{
		\psfrag{a}[c][c]{$u_1$}
		\psfrag{b}[c][c]{\rotatebox[origin=c]{90}{$u_2$}}
		\psfrag{c}[c][c]{$\mathlarger{\mathlarger{x_0}}$}
		\psfrag{u}[c][c]{$\mathlarger{\mathlarger{u_{\subRL}}}$}}
	\vspace{10pt}
	\caption{Steps for action projection using parameterized reachability analysis, where the reachable set is depicted in gray and the unsafe regions are shown in red: 1) Computation of the reachable set for all actions starting from the current state \captionmath{x_0}. 2) Extraction of action constraints from the intersections between the reachable set and unsafe regions. 3) Projection of the action \captionmath{u_{\subRL}} outputted by the agent to the closest safe action.} 
	\label{fig:overview}
\end{figure}

\subsection{State of the Art}

We first provide a summary of the current state of the art in safety-related methods of reinforcement learning.
The term \textit{safe reinforcement learning} refers to approaches that aim to obtain safe agents, but do not provide hard safety guarantees.
One example for this is constrained reinforcement learning \cite{Achiam2017,Yang2019}, where the objective of the training phase is to maximize the reward while satisfying safety constraints. While advantages of this technique are that no system model is required and that even complex temporal logic safety specifications 
\cite{Hasanbeig2020,Wang2022} can be considered, the obvious disadvantage is that hard safety guarantees can be provided during neither training nor deployment. The same is true for probabilistic approaches \cite{Konighofer2021,Geibel2005} that aim to identify the safety probability of an action. Overall, safe reinforcement learning techniques can be used for non-critical applications, where unsafe actions do not cause major damage; however, these methods are not suited for safety-critical systems.

In contrast to safe reinforcement learning, \textit{provably safe reinforcement learning} approaches provide hard safety guarantees. They can be divided into the three main categories: \textit{action masking}, \textit{action replacement}, and \textit{action projection} \cite{Krasowski2022}. In action masking \cite{Krasowski2020,Isele2018}, a mask that only allows the agent to choose actions from the set of safe actions is applied. One disadvantage of this method is that it is often hard to explicitly compute the set of safe actions, especially for continuous action spaces, where the set of safe actions often has a very complex non-convex shape as shown in Fig.~\ref{fig:overview}. In addition, it is non-trivial to correctly consider the masking during training so that the reinforcement learning algorithm is not perturbed \cite{Huang2022}.
For action replacement \cite{Thumm2022,Saunders2018,Hunt2021,Alshiekh2018}, unsafe actions returned by the agent are replaced by safe actions. As replacement, one can either use a single safe action obtained from a failsafe planner \cite{Thumm2022} or via human feedback \cite{Saunders2018} or one can sample from the set of safe actions \cite{Hunt2021,Alshiekh2018}. Also the well-known simplex architecture \cite{Seto1998,Phan2020,Schurmann2021}, where a safe controller is used as a backup for an unsafe controller, can be categorized as action replacement. One disadvantage of action replacement is that the difference between the original action and the replacement action can be very large, which might prevent the agent from completing its task. Action projection tries to avoid this issue by finding the safe action that is closest to the action suggested by the agent.

Since our approach applies action projection, we discuss this category in more detail. The most prominent methods for action projection are \textit{control barrier functions} \cite{Cheng2019,Marvi2021}, \textit{model predictive control} \cite{Bastani2021,Wabersich2021}, and \textit{parameterized reachability analysis} \cite{Shao2021}. A control barrier function is a level-set function that divides the state space into a safe and a potentially unsafe region. Here, action projection is formulated as an optimization problem, where the correction of the action is minimized, such that the system stays inside the safe region defined by the control barrier function.
While an advantage is that control barrier functions can for static environments guarantee safety for infinite time, the method also has several disadvantages: 1) It is often not easy to find a suitable control barrier function, especially in the presence of dynamic obstacles. 2) Control barrier functions are often quite conservative since they usually exclude many states that are safe. 3) The approach is often limited to control affine systems because the optimization problems would otherwise become non-convex. 4) It is challenging to consider input constraints as well as process noise and measurement errors. The second method is model predictive control, which also formulates the projection as an optimization problem, but uses the safety constraint that the system should not enter any unsafe regions for a certain finite prediction horizon, which avoids the requirement for a control barrier function.
However, one downside is that it is often not possible to guarantee that the solution is robustly safe despite process noise and measurement errors since for nonlinear systems these uncertainties usually cannot be encoded directly into the optimization problem. Our safety shield is based on the parameterized reachability analysis approach, which is visualized in Fig.~\ref{fig:overview}: The first step is to compute the reachable set for all available actions. Since this reachable set is parameterized by the actions, one can directly extract the safety constraints for action projection from the intersection between the reachable sets and the unsafe regions. Since process noise as well as measurement errors can conveniently be integrated into reachability analysis, this approach is very well suited for guaranteeing robust safety.

Due to its advantageous properties, several approaches apply reachability analysis to guarantee safety. One method \cite{Gillula2012} uses the Hamilton-Jacobi reachability framework \cite{Mitchell2005} to compute the backward reachable set starting from the unsafe sets --- a state is safe for all possible actions if it is outside of the backward reachable set. This has the disadvantage that for each unsafe set a different backward reachable set has to be computed. Moreover, the Hamilton-Jacobi framework requires gridding the state space so that the computational complexity of the approach grows exponentially with the system dimension. Another method \cite{Selim2022} applies reachability analysis for black-box systems and uses a differentiable collision check that is based on constrained zonotopes \cite{Scott2016} to efficiently push the reachable set for the proposed action away from unsafe sets. This, however, has the drawback that the reachable set has to be recomputed after each correction update of the action, which is computationally demanding. The method closest to our approach is a \textit{reachability-based trajectory safeguard} \cite{Shao2021}, which computes the parameterized reachable set for a simplified trajectory-generating model and determines a safe action satisfying the constraints extracted from the reachable set via random sampling. While this approach can be computationally efficient for some systems, sampling methods often fail to find feasible solutions, especially in high-dimensional action spaces.

\subsection{Contributions and Outline}

We present a novel safety shield that is based on action projection using parameterized reachability analysis. 
This safety shield extends our previous work on dependency preserving reachability analysis \cite{Althoff2013a,Kochdumper2020c,Kochdumper2019} by a method for correcting unsafe actions, and we additionally also study the effect online verification has on the learning process. 
Unlike the related approach in \cite{Shao2021}, our safety shield directly operates on the original nonlinear system model rather than on a simplified trajectory-generating model. Moreover, in contrast to \cite{Shao2021}, we use conservative polynomialization \cite{Althoff2013a} instead of conservative linearization \cite{Althoff2008c} for reachability analysis, which enables us to efficiently capture the nonlinear effects the actions have on the system. Another advantage over \cite{Shao2021} is that we use mixed-integer optimization instead of random sampling for projection, which always finds the action with the smallest correction. Finally, the various design choices provided by our safety shield enable the user to fine-tune its performance for the considered application.

The remainder of this paper is structured as follows: After introducing some preliminaries in Sec.~\ref{sec:prelim}, we provide the problem definition in Sec.~\ref{sec:problem}. Our main contribution is the reachability-based safety shield for reinforcement learning presented in Sec.~\ref{sec:main}, for which we discuss several extensions in Sec.~\ref{sec:impr}. 
Finally, we demonstrate our approach on several numerical examples in Sec.~\ref{sec:exp} and conclude with a discussion of its properties in Sec.~\ref{sec:disc}. 

\section{Preliminaries}
\label{sec:prelim}

We first introduce our notation and define the set representations that we use in this paper.

\subsection{Notation}

Sets are denoted by calligraphic letters, matrices by uppercase letters, and vectors by lowercase letters. Given a vector $a \in \R^n$, $a_{(i)}$ is the $i$-th entry and the p-norm is denoted by $\|a \|_p$. Given a matrix $A \in \mathbb{R}^{n \times m}$, $A_{(i,\cdot)}$ represents the $i$-th matrix row, $A_{(\cdot,j)}$ the $j$-th column, and $A_{(i,j)}$ the $j$-th entry of matrix row $i$. The concatenation of two matrices $C$ and $D$ is denoted by $[C~D]$, $I_n \in \R^{n \times n}$ is the identity matrix, and the symbols $\mathbf{0}$ and $\mathbf{1}$ represent vectors of zeros and ones of proper dimension. We further introduce an $n$-dimensional interval as $\mathcal{I} := [\underline{x},\overline{x}],~ \forall i ~ \underline{x}_{(i)} \leq \overline{x}_{(i)},~ \underline{x},\overline{x} \in \mathbb{R}^n$. Given two sets $\mathcal{S}_1,\mathcal{S}_2 \subset \Rn$, their Minkowski sum is $\mathcal{S}_1 \oplus \mathcal{S}_2 = \{s_1 + s_2~|~ s_1 \in \mathcal{S}_1,~s_2 \in \mathcal{S}_2 \}$ and their Cartesian product is $\mathcal{S}_1 \times \mathcal{S}_2 = \big \{[s_1^T~s_2^T]^T~\big|~ s_1 \in \mathcal{S}_1,~s_2 \in \mathcal{S}_2 \big\}$.

\subsection{Set Representations}

Our approach relies on several different set representations, which we introduce here. Let us begin with polytopes, for which we use the halfspace representation: 
\begin{definition}[Polytope]
	Given a constraint matrix $A \in \R^{\polyCons \times n}$ and a constraint offset $b \in \R^\polyCons$, the halfspace representation of a polytope $\mathcal{P} \subseteq \Rn$ is\begin{equation*}
		\mathcal{P} := \big\{ x \in \Rn~\big|~ A \,x \leq b \big\}.
	\end{equation*}
	We use the shorthand $\mathcal{P} = \langle A, b\rangle_P$.
\end{definition}
Zonotopes are a special type of polytopes that can be represented efficiently using generators:
\begin{definition}[Zonotope] \label{def:zonotope}
	Given a center vector $c \in \Rn$ and a generator matrix $G \in \R^{n \times p}$, a zonotope $\mathcal{Z} \subset \Rn$ is 
	\begin{equation*}
		\mathcal{Z} := \bigg \{ c + \sum_{i=1}^p G_{(\cdot,i)} \, \alpha_i ~ \bigg| ~ \alpha_i \in [\shortminus 1,1] \bigg\}
	\end{equation*}
	with so-called factors $\alpha_i$. We use the shorthand $\mathcal{Z} = \langle c,G\rangle_Z$.
\end{definition}
An extension to zonotopes are polynomial zonotopes \cite{Althoff2013a}, which can represent non-convex sets. We use the sparse representation of polynomial zonotopes \cite{Kochdumper2019}\footnote{In contrast to \cite[Def.~1]{Kochdumper2019} we do not integrate the constant offset $c$ into $G$. Moreover, we omit the identifier vector used in \cite{Kochdumper2019} for simplicity}:
\begin{definition}[Polynomial Zonotope] \label{def:polynomialZonotope}
	Given a constant offset $c \in \Rn$, a generator matrix of dependent generators $G \in \mathbb{R}^{n \times h}$, a generator matrix of independent generators $G_I \in \mathbb{R}^{n \times q}$, and an exponent matrix $E \in \mathbb{N}_{0}^{p \times h}$, a polynomial zonotope $\mathcal{PZ} \subset \Rn$ is  
	\begin{equation*}
		\begin{split}
			\mathcal{PZ} := \bigg\{ c + \sum _{i=1}^h \bigg( \prod _{k=1}^p \alpha _k ^{E_{(k,i)}} \bigg) G_{(\cdot,i)} + & \sum _{j=1}^{q} \beta _j G_{I(\cdot,j)}  \\
			& \bigg|~ \alpha_k, \beta_j \in [\shortminus 1,1] \bigg\}.
		\end{split}
	\end{equation*}
	The scalars $\alpha_k$ are called dependent factors and $\beta_j$ independent factors. We use the shorthand $\mathcal{PZ} = \langle c, G, G_I, E \rangle_{PZ}$.
\end{definition}
Polynomial zonotopes can equivalently represent intervals, zonotopes, polytopes, and Taylor models \cite[Sec.~II.B]{Kochdumper2019}. Moreover, due to their polynomial nature, they are closely related to polynomial level sets:
\begin{definition}[Polynomial Level Set] 
	Given a vector of coefficients $a \in \R^h$, an offset $b \in \R$, and an exponent matrix $E \in \mathbb{N}_{0}^{n \times h}$, a polynomial level set $\mathcal{LS} \subseteq \Rn$ is
	\begin{equation*}
		\mathcal{LS} := \bigg \{ x \in \Rn ~\bigg |~ \sum_{i=1}^h \bigg( \prod _{k=1}^n x_{(k)} ^{E_{(k,i)}} \bigg) a_{(i)} \leq b  \bigg \}.
	\end{equation*}
	We use the shorthand $\mathcal{LS} = \langle a,b,E \rangle_{LS}$.
\end{definition}

\section{Problem Formulation}
\label{sec:problem}

We consider general nonlinear disturbed systems with input constraints defined by the ordinary differential equation
\begin{equation}
	\dot x(t) = f\big(x(t),u(t),w(t)\big),
	\label{eq:system}
\end{equation}
where $x(t) \in \Rn$ is the system state, $u(t) \in \R^m$ is the control input, $w(t) \in \R^\numDist$ is the process noise, $f:~\Rn \times \R^m \times \R^\numDist \to \Rn$ is a Lipschitz continuous function, and $t \in \R^+$ is time. The process noise is bounded by a compact set $w(t) \in \mathcal{W} \subset \R^\numDist$ and the system has to satisfy the input constraints defined by the convex set $u(t) \in \mathcal{U} \subseteq \R^m$. The set $\mathcal{W}$ can for example be determined from measurements of the real physical system using conformance checking \cite{Roehm2018a}.

Given a nonlinear system defined as in \eqref{eq:system}, the goal is to solve a reach-avoid problem, where the system state should be steered from the current state $x_0 = x(0)$ to a goal set $\mathcal{G} \subseteq \Rn$ while avoiding collisions with potentially time-varying unsafe sets $\mathcal{F}_i \subset \Rn$, $i = 1,\dots,\numObs$, where $\numObs$ denotes the number of unsafe sets. In case the measurements of the system state are subject to a measurement error $v(t) \in \mathcal{V}$, the goal becomes to steer all states in the set $x_0 \oplus \mathcal{V}$ to the goal set.
We aim to solve reach-avoid problems with reinforcement learning, where we train an agent to return the control inputs $u_a(t)$ for a given state $x(t)$ steering the system to the goal set while avoiding obstacles. However, we have no guarantee that the behavior learned by the agent is safe. Therefore, we add a safety shield that is based on reachability analysis to obtain formal guarantees:
\begin{definition}[Reachable Set] 
	Let $\xi(t,x_0,u(\cdot),w(\cdot))$ denote the solution of \eqref{eq:system} at time $t$ for an initial state $x_0 = x(0)$, control input trajectory $u(\cdot)$ and process noise trajectory $w(\cdot)$. The reachable set at time $t$ is
	\begin{equation*}
		\begin{split}
			\mathcal{R}(t) := \big \{ \xi \big(t,x_0,u(\cdot),w(\cdot)\big)~\big |~&x_0 \in \mathcal{X}_0, \\ 
			& \forall \tau \in [0,t]:~w(\tau) \in \mathcal{W} \big \},
		\end{split}
	\end{equation*} 
	where $\mathcal{X}_0 \subset \Rn$ is the initial set and $\mathcal{W} \subset \R^\numDist$ is the set of process noise.
	\label{def:reachSet}
\end{definition}
For our safety shield, we consider that $\mathcal{U}$, $\mathcal{W}$, and $\mathcal{V}$ are represented as zonotopes, and $\mathcal{G}$ and $\mathcal{F}_i$ are represented as polytopes in halfspace representation. Moreover, we use polynomial zonotopes to represent reachable sets. In case other agents are present in the environment, we can apply set-based methods \cite{Koschi2020} to safely predict their future behavior and obtain the corresponding time-varying unsafe sets.

\section{Safety Shield}
\label{sec:main}

As visualized in Fig.~\ref{fig:overview}, the high-level idea behind our safety shield is to compute the reachable set for a time horizon of $t_f$ and the set of all control inputs satisfying the input constraints $\forall t \in [0,t_f]: \, u(t) \in \mathcal{U}$ rather than a single control input trajectory $u(\cdot)$. The intersection of this reachable set with the unsafe sets then yields constraints that define safe control inputs, which we can use to formulate the projection of the control input $u_{\subRL}$ provided by the reinforcement policy to the closest safe control input as an optimization problem. We first consider input trajectories that are constant over time for simplicity and discuss more advanced control strategies later in Sec.~\ref{subsec:controlLaw}. For constant control inputs, we can compute the reachable set using the extended system dynamics 
\begin{equation}
	\begin{bmatrix} \dot x(t) \\ \dot u(t) \end{bmatrix} = \begin{bmatrix} f\big( x(t),u(t),w(t) \big) \\ \mathbf{0} \end{bmatrix}
	\label{eq:extSys}
\end{equation}
together with the initial set $\mathcal{X}_0 = x_0 \times \mathcal{U}$, where we omit the set of measurement errors $\mathcal{V}$ for simplicity.
For reachability analysis, we use the conservative polynomialization algorithm \cite{Althoff2013a}, which encloses the nonlinear dynamics in \eqref{eq:system} by a differential inclusion $\dot x \in p(x(t),u(t),w(t)) \oplus \mathcal{E}$ consisting of a polynomial approximation $p(x(t),u(t),w(t))$ and the abstraction error $\mathcal{E}$. This reachability algorithm explicitly preserves dependencies between the initial states and the reachable states \cite{Kochdumper2020c}. Since with the extended system dynamics in \eqref{eq:extSys}, the control inputs become part of the system state, we can therefore directly determine from the reachable set which control inputs steer the system to unsafe regions. Let us demonstrate this dependency preservation by an example:   
\begin{example}
	As a running example we consider the system
	\begin{equation*}
		\begin{split}
			\dot x_1 = 4 + 2 \, x_2 \, u_1 + w_1, ~ \dot x_2 = 1.7 + u_1 \, u_2
		\end{split}
	\end{equation*}
	with initial state $x_0 = [0~0]^T$, set of process noise $\mathcal{W} = [-0.01, 0.01]$, and time horizon $t_f = 1\,\si{\second}$. Moreover,  
	\begin{equation*}
		\mathcal{F} = \big \langle \big[[\shortminus 4 ~\shortminus 1]^T ~[\shortminus 1 ~ \shortminus 4]^T \, \big], [\shortminus 14 ~ \shortminus 8]^T \big \rangle_P
	\end{equation*}
	is the unsafe set and 
	\begin{equation*}
		\mathcal{U} = \bigg \{ \begin{bmatrix} \shortminus 0.5 \\ 1 \end{bmatrix} + \begin{bmatrix}  0.5 \\ 0 \end{bmatrix} \alpha_1 + \begin{bmatrix} 0 \\  1 \end{bmatrix} \alpha_2 ~ \bigg| ~ \alpha_1,\alpha_2 = [\shortminus1, 1] \bigg \}
	\end{equation*}
	is the set of control inputs. With the conservative polynomialization algorithm we obtain the final reachable set
	\begin{equation*}
		\begin{split}
			& \mathcal{R}(t_f) = \bigg \{ \hspace{-2pt} \begin{bmatrix} 3.4 \\ 1.2 \end{bmatrix}\hspace{-2pt} + \hspace{-2pt}\begin{bmatrix} 0.34 \\ 0.5 \end{bmatrix}\hspace{-2pt} \alpha_1 + \hspace{-2pt}\begin{bmatrix} 0.25 \\ \shortminus 0.5 \end{bmatrix}\hspace{-2pt} \alpha_2 + \hspace{-2pt}\begin{bmatrix} \shortminus 0.49 \\ 0.5 \end{bmatrix}\hspace{-2pt} \alpha_1 \alpha_2 \, + \\
			& \begin{bmatrix} 0.25 \\ 0 \end{bmatrix} \hspace{-2pt} \alpha_1^2 + \hspace{-2pt}\begin{bmatrix} 0.25 \\ 0 \end{bmatrix} \hspace{-2pt} \alpha_1^2 \alpha_2 + \hspace{-2pt} \begin{bmatrix} 0.1 \\ 0 \end{bmatrix}\hspace{-2pt} \beta_1 \,\bigg| \, \alpha_1,\alpha_2,\beta_1 \in [\shortminus 1,1] \bigg\},
		\end{split}
	\end{equation*}
	which is visualized in Fig.~\ref{fig:runningExample}. Since the reachable set $\mathcal{R}(t_f)$ and the input set $\mathcal{U}$ are parameterized by the same factors $\alpha_1$ and $\alpha_2$, we have a direct analytical relation between the control inputs and the corresponding reachable states.
	\label{ex:runningExample}
\end{example}

\begin{figure}[!tb]
	\centering
	\setlength{\belowcaptionskip}{-15pt}
	\psfragfig[width=0.99\columnwidth]{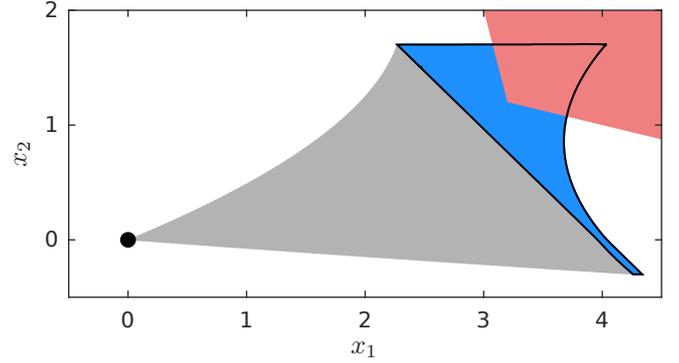}{
		\psfrag{a}[c][c]{$x_1$}
		\psfrag{b}[c][c]{\rotatebox[origin=c]{180}{$x_2$}}}
	\caption{Reachable set for the system from Example~\ref{ex:runningExample}, where the initial state \captionmath{x_0} is shown as a black dot, the final reachable set \captionmath{\mathcal{R}(t_f)} is depicted in blue with a black border, and the unsafe set \captionmath{\mathcal{F}} is shown in red.}
	\label{fig:runningExample}
\end{figure}

We now exploit the analytical relation between the control inputs and the reachable states to determine the set of safe control inputs. As demonstrated in the example above, the control input $u(t) \in \mathcal{U} = \langle c_u,G_u\rangle_Z$ is unambiguously defined by the factors $\alpha$ via the relation $u = c_u + G_u \alpha$ through the definition of a zonotope in Def.~\ref{def:zonotope}. Instead of determining the set of safe control inputs directly, we therefore determine the safe set for $\alpha$ instead, since this simplifies the computations as it becomes apparent later. The independent generators of the polynomial zonotope $\mathcal{R}(t_f)$ represent uncertainties that results from abstraction errors during reachability analysis as well as from the process noise. Consequently, a control input is safe only if the reachable set does not intersect the unsafe sets for any possible value of the independent factors $\beta_j$. We formulate this in the following theorem, which extends our previous results for unsafe sets given as halfspaces \cite[Sec.~4.1]{Kochdumper2020c} to the more general case of polytopes:

\begin{theorem} \label{prop:intersect}
	Given is an unsafe set $\mathcal{F} = \langle A,b \rangle_P \subset \Rn$ consisting of $\polyCons$ halfspace constraints and the reachable set $\mathcal{R}(t) = \langle c,G,G_I,E \rangle_{PZ} \subset \R^n$ of the system in \eqref{eq:extSys} computed with the conservative polynomialization algorithm \cite{Althoff2013a} for the initial set $\mathcal{X}_0 = x_0 \times \mathcal{U}$, $x_0 \in \Rn$, $\mathcal{U} = \langle c_u,G_u\rangle_Z \subset \R^m$ and the set of process noise $\mathcal{W} \subset \R^{\numDist}$. The following constraints on the zonotope factors $\alpha$ that parameterize the control input ensure that there exists no trajectory that enters the unsafe set:
	\begin{equation*}
		\begin{split}
			\forall \alpha \in [\shortminus \mathbf{1},\mathbf{1}] \cap \bigcup_{l=1}^{\polyCons} \mathcal{LS}_l, ~ & \forall w(\cdot) \in \mathcal{W}: \\[-10pt]
			&  ~~ \xi(t,x_0,c_u + G_u \alpha,w(\cdot)) \notin \mathcal{F}
		\end{split}
	\end{equation*}
	with
	\begin{equation*}
		\mathcal{LS}_l = \bigg \langle \hspace{-5pt}\shortminus A_{(l,\cdot)} G, A_{(l,\cdot)} c - \sum_{j=1}^q \big|A_{(l,\cdot)} G_{I(\cdot,j)}\big|- b_{(l)} , E \bigg \rangle_{LS}
	\end{equation*}
	for $l = 1,\dots,\polyCons$.
\end{theorem}
\begin{proof}
	A single point $x \in \Rn$ is located outside the unsafe set $\mathcal{F}$ if it is fully located outside of at least one halfspace: 
	\begin{equation}
		\bigvee_{l=1}^{\polyCons} A_{(l,\cdot)} \,x  > b_{(l)} \Rightarrow x \not\in \mathcal{F}.
		\label{eq:proof1}
	\end{equation}
	Moreover, due to dependency preservation of reachability analysis, it holds according to \cite[Thm.~1]{Kochdumper2020c} that the disturbed trajectory $\xi(t,x_0,c_u + G_u \alpha,w(\cdot))$ for a specific control input $u = c_u + G_u \alpha$ is contained inside the reachable subset obtained by restricting the factors $\alpha_k \in [\shortminus 1,1]$ in the definition of polynomial zonotopes in Def.~\ref{def:polynomialZonotope} to the corresponding concrete value for $\alpha = [\alpha_1~\dots~\alpha_p]^T$:
	\begin{equation*}
		\begin{split}
			& \forall \alpha_k \in [\shortminus 1,1]: \, \xi(t,x_0,c_u + G_u \alpha,w(\cdot)) \in  \\
			& c + \sum _{i=1}^h \bigg( \prod _{k=1}^p \alpha _k ^{E_{(k,i)}} \hspace{-3pt}\bigg) G_{(\cdot,i)} + \bigg\{ \hspace{-3pt} \sum _{j=1}^{q} \beta _j G_{I(\cdot,j)} \, \bigg| \, \beta_j \in [\shortminus 1,1] \bigg\}.
		\end{split}
		\label{eq:proof2}
	\end{equation*}
	Finally, combining this with \eqref{eq:proof1} under the consideration that the constraints should hold for all values of the independent factors $\beta_j$ yields 
	\begin{equation*}
		\begin{split}
			& \forall \alpha_k,\beta_j \in [\shortminus 1,1]: \bigvee_{l=1}^{\polyCons} A_{(l,\cdot)} c + \sum _{i=1}^h \hspace{-3pt} \bigg( \prod _{k=1}^p \alpha _k ^{E_{(k,i)}} \hspace{-3pt}\bigg) A_{(l,\cdot)}G_{(\cdot,i)} \\
			& + \hspace{-1pt}\sum _{j=1}^{q} \beta _j A_{(l,\cdot)}G_{I(\cdot,j)} \hspace{-1pt}> b_{(l)} \hspace{-1pt}\Rightarrow \xi(t,x_0,c_u \hspace{-1pt}+\hspace{-1pt} G_u \alpha,w(\cdot)) \hspace{-1pt}\not\hspace{-1pt} \in \mathcal{F},
		\end{split}
	\end{equation*}
	which results in the statement of the theorem after bringing the constant offset and the independent generators to the other side of the inequality.
\end{proof}

\begin{remark}
	A geometric interpretation of Thm.~\ref{prop:intersect} is that we first bloat the obstacle $\mathcal{F}$ by the uncertainty given by the independent generators through pushing each polytope halfspace outward. Next, we obtain the constraints via intersecting with the part of the polynomial zonotope spanned by the dependent generators, where the intersection between each halfspace of the bloated polytope $\mathcal{F}$ corresponds to a polynomial level set constraint for the factors $\alpha$. 
\end{remark}

Thm.~\ref{prop:intersect} defines a feasible region $\alpha \in [\shortminus \mathbf{1},\mathbf{1}] \cap \bigcup_l \mathcal{LS}_l$ for the factors $\alpha$ that parameterize the control input such that the intersection between a reachable set at a specific point in time and a single unsafe set is empty. However, to guarantee safety we have to consider the reachable set for the whole time horizon $t \in [0, t_f]$, which consists of a sequence of reachable sets $\mathcal{R}(\tau_0),\mathcal{R}(\tau_1), \dots, \mathcal{R}(\tau_f)$ for consecutive time intervals $\tau_0,\tau_1, \dots, \tau_f$. Moreover, we might also have more than one unsafe set. So overall we obtain one feasible region $\alpha \in [\shortminus \mathbf{1},\mathbf{1}] \cap \bigcup_l \mathcal{LS}_l$ for each pair of reachable sets and unsafe sets resulting in an intersection. The feasible region for $\alpha$ to guarantee safety for all time intervals and all unsafe sets is given by the intersection of the feasible regions for single pairs:
\begin{equation*}
	\alpha \in [\shortminus \mathbf{1},\mathbf{1}] \cap\bigcap_{\ind=1}^{\numInt} \bigcup_{l=1}^{\polyCons_{\ind{}}} \underbrace{\langle a_{\ind{}l}, b_{\ind{}l}, E_{\ind{}l} \rangle_{LS}}_{\mathcal{LS}_{\ind{}l}},
\end{equation*}
where the level sets $\mathcal{LS}_{\ind{}l}$ are obtained from Thm.~\ref{prop:intersect} and $\numInt$ is the number of intersecting pairs. To efficiently check if a reachable set represented by a polynomial zonotope intersects an obstacle represented by a polytope, the polynomial zonotope refinement algorithm \cite{Bak2022} can be used. This algorithm recursively splits the polynomial zonotope along the longest generator until the intersection with the polytope can either be proven or disproven using zonotope enclosures of the split polynomial zonotopes. Overall, given a vector of factors $\alpha_{\subRL} \in \R^p$ that corresponds to the control input $u_{\subRL} = c_u + G_u \alpha_{\subRL} \in \mathcal{U} = \langle c_u,G_u\rangle_Z$ provided by the reinforcement learning policy, we can formulate the projection to the closest safe control input as an optimization problem:
\begin{equation*}
	\min_{\alpha \in [\shortminus \mathbf{1},\mathbf{1}]} \| \alpha - \alpha_{\subRL} \|_2^2 ~~~ \text{s.t.} ~~~ \alpha \in \bigcap_{\ind{}=1}^{\numInt} \bigcup_{l=1}^{\polyCons_\ind{}} \langle a_{\ind{}l}, b_{\ind{}l}, E_{\ind{}l} \rangle_{LS}.
\end{equation*}
This is a disjunctive programming problem, which can be formulated as a mixed-integer quadratic program with polynomial constraints using the convex hull relaxation \cite{Grossmann2003}:
\begin{equation}
	\min_{\alpha \in [\shortminus \mathbf{1},\mathbf{1}]} \| \alpha - \alpha_{\subRL} \|_2^2
	\label{eq:optimize}
\end{equation}
subject to
\begin{equation*}
	\begin{split}
		& ~\sum_{i=1}^h \bigg( \prod_{k=1}^p \alpha_{\ind{}l(k)}^{E_{\ind{}l(k,i)}} \Big) a_{\ind{}l(\cdot,k)} \lambda_{\ind{}l} \leq \shortminus b_{\ind{}l}\,\lambda_{\ind{}l}, \\
		&\lambda_{\ind{}l} \in \{0,1\}, ~\alpha = \sum_{l=1}^{\polyCons_\ind{}} \lambda_{\ind{}l} \, \alpha_{\ind{}l}, ~ \sum_{l=1}^{\polyCons_\ind{}} \lambda_{\ind{}l} = 1, 
	\end{split}
\end{equation*}
for $\ind{} = 1,\dots,\numInt$ and $l = 1,\mydots,\polyCons_\ind{}$. Here, the disjunction is realized using the binary variables $\lambda_{\ind{}l} \in \{0,1\}$ which modify the corresponding polynomial constraints to be either active ($\lambda_{\ind{}l} = 1$) or inactive ($\lambda_{\ind{}l} = 0$). Let us demonstrate the optimization for our running example:
\begin{example} \label{ex:optimization}
	As shown in Fig.~\ref{fig:runningExample}, for the nonlinear system in Example~\ref{ex:runningExample} only the final reachable set $\mathcal{R}(t_f)$ intersects the unsafe set $\mathcal{F}$. We consequently obtain the feasible region for $\alpha$ by applying Thm.~\ref{prop:intersect} to the sets $\mathcal{R}(t_f)$ and $\mathcal{F}$, which yields $\alpha \in \mathcal{LS}_1 \vee \mathcal{LS}_2$ with
	\begin{equation*}
		\begin{split}
			\mathcal{LS}_1 = \big \{ [\alpha_1~\alpha_2]^T \in \R^2 \, \big| \, 1.86 \, \alpha_1 &+  0.5 \, \alpha_2 - 1.46 \, \alpha_1 \alpha_2  \\
			& + \alpha_1^2 + \alpha_1^2 \alpha_2 \leq \shortminus 1.2 \big \} 
		\end{split}
	\end{equation*}
	and
	\begin{equation*}
		\begin{split}
			\mathcal{LS}_2 = \big \{ [\alpha_1~\alpha_2]^T  \in \R^2 \, \big| &\, 2.34 \, \alpha_1 - 1.75 \, \alpha_2 + 1.51 \, \alpha_1 \alpha_2 \\ 
			& + 0.25 \, \alpha_1^2 + 0.25 \, \alpha_1^2 \alpha_2 \leq \shortminus 0.3 \big \}.
		\end{split} 
	\end{equation*}
	Given $\alpha_{\subRL} = [0.3~0]^T$, the optimization problem \eqref{eq:optimize} becomes
	\begin{equation*}
		\min_{\alpha_1,\alpha_2 \in [\shortminus 1,1]} (\alpha_1 - 0.3)^2 \, + \, \alpha_2^2 
	\end{equation*}
	subject to
	\begin{eqnarray*}
		& \hspace{-1cm}1.86 \, \alpha_{11(1)}\lambda_{11} + 0.5 \, \alpha_{11(2)}\lambda_{11} - 1.46 \, \alpha_{11(1)} \alpha_{11(2)}\lambda_{11}   \\
		& \hspace{1.7cm} + \alpha_{11(1)}^2 \lambda_{11}+ \alpha_{11(1)}^2 \alpha_{11(2)} \lambda_{11}\leq \shortminus 1.2 \, \lambda_{11},\\[5pt]
		& \hspace{-0.8cm} 2.34 \, \alpha_{12(1)}\lambda_{12} - 1.75 \, \alpha_{12(2)} \lambda_{12}+ 1.51 \, \alpha_{12(1)} \alpha_{12(2)}\lambda_{12} \\
		& \hspace{0.7cm} + 0.25 \, \alpha_{12(1)}^2\lambda_{12} + 0.25 \, \alpha_{12(1)}^2 \alpha_{12(2)}\lambda_{12} \leq \shortminus 0.3 \, \lambda_{12},\\[5pt]
		& \lambda_{11},\lambda_{12} \in \{0,1\}, ~~  \lambda_{11} + \lambda_{12} = 1 \\[5pt]
		& \begin{bmatrix} \alpha_1 \\ \alpha_2 \end{bmatrix} = \lambda_{11} \begin{bmatrix} \alpha_{11(1)} \\ \alpha_{11(2)} \end{bmatrix} + \lambda_{12} \begin{bmatrix} \alpha_{12(1)} \\ \alpha_{12(2)} \end{bmatrix},
	\end{eqnarray*}
	which has the optimal solution $\alpha = [\alpha_1~\alpha_2]^T = [0.04~0.2]^T$. The feasible regions for $\alpha_1$ and $\alpha_2$ are shown in Fig.~\ref{fig:optimization}.
\end{example}

In the presence of measurement errors $v(t) \in \mathcal{V}$ we can apply the same overall approach but have to change the initial set to $\mathcal{X}_0 = (x_0 \oplus \mathcal{V}) \times \mathcal{U}$, where the set $\mathcal{V}$ has to be represented by independent generators to ensure that safety is guaranteed for all possible values of the measurement errors. While we focused on the conservative polynomialization algorithm \cite{Althoff2013a} for simplicity, our safety shield is also compatible with other reachability approaches as long as they preserve dependencies between initial states and reachable states. This is for example the case for algorithms that compute reachable sets using the Picard-Lindel\"of iteration together with Taylor models \cite{Chen2012}. 

The safety shield can be used during reinforcement learning or for a learned agent. For every decision step, the action suggested by the agent is corrected to the closest safe action by \eqref{eq:optimize} only if it violates safety constraints. If the safety shield is used during learning, it can be beneficial to adapt the reward to inform the agent about corrections of actions \cite{Krasowski2022}.

\begin{figure}[!tb]
	\centering
	\setlength{\belowcaptionskip}{-15pt}
	\psfragfig[width=0.98\columnwidth]{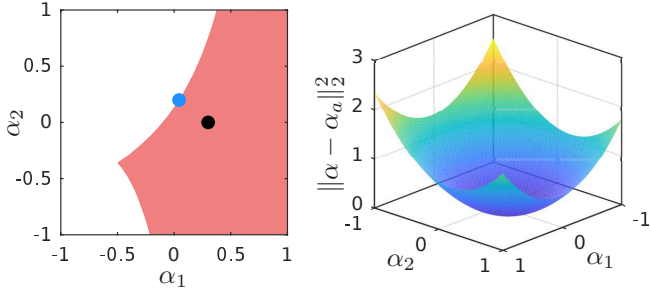}{
		\psfrag{a}[c][c]{$\alpha_1$}
		\psfrag{b}[c][c]{\rotatebox[origin=c]{180}{$\alpha_2$}}
		\psfrag{c}[c][c]{\rotatebox[origin=c]{180}{$\|\alpha - \alpha_{\subRL} \|_2^2$}}
		\psfrag{d}[c][c]{$\alpha_2$}}
	\caption{Domain (left) and objective function (right) for the optimization problem from Example~\ref{ex:optimization}. For the domain plot the set of infeasible values is shown in red, the desired value \captionmath{\alpha_{\subRL}} is visualized as a black dot, and the optimal solution to the optimization problem is depicted as a blue dot.} 
	\label{fig:optimization}
\end{figure}

\section{Extensions}
\label{sec:impr}
We now discuss several extensions for our safety shield.

\subsection{Different Types of Control Laws}
\label{subsec:controlLaw}

For the basic safety shield presented in Sec.~\ref{sec:main}, for simplicity we considered that the control input is kept constant for the whole planning horizon. Since this is very restrictive and would in practice often prevent us from finding a feasible solution, we now discuss how to realize more advanced control strategies. Note that the reinforcement learning agent has to match the control law used for the safety shield.

\paragraph{Piecewise Constant Control Law}
\label{para:piecewise}

One simple but very effective extension to constant control inputs are piecewise constant control inputs. Instead of determining a single control input from the input set $\mathcal{U}$, we determine control inputs for all piecewise constant segments from the set $\mathcal{U} \times \dots \times \mathcal{U}$. We can still use the extended system in \eqref{eq:extSys}, but have to reset the initial set for reachability analysis to $\mathcal{R}(t_i) \times \mathcal{U}$ after each of the $i = 1,\dots,\pieces$ piecewise constant time segments $[t_{i-1},t_i]$ with $t_i = i \cdot t_f/\pieces$, where $\mathcal{R}(t_i)$ is the final reachable set from the previous segment.    

\paragraph{Polynomial Control Law}

Another possibility is to use control laws that are polynomial functions with respect to time. We consider the quadratic case for simplicity since the extension to general polynomials is straightforward. For a quadratic control law $u(t) = c_{(1)} + c_{(2)} t + c_{(3)} t^2$ parameterized by the vector of coefficients $c \in \R^3$, we can use the extended system
\begin{equation*}
	\begin{bmatrix} \dot x(t) \\ \dot c \\ \dot t \end{bmatrix} = \begin{bmatrix} f\big( x(t),c_{(1)} + c_{(2)} t + c_{(3)} t^2,w(t) \big) \\ \mathbf{0} \\ 1 \end{bmatrix}
\end{equation*}
together with the initial set $x_0 \times \mathcal{C} \times 0$. In the optimization problem \eqref{eq:optimize} we then determine the values for the parameter vector $c$, where we add the constraint $c_{(1)} + c_{(2)} t + c_{(3)} t^2 \in \mathcal{U}$ to ensure that the input constraints are satisfied. The initial set $\mathcal{C} \subset \R^3$ for the coefficient vector $c$ can be determined by estimating the feasible values for $c$ such that the constraint $c_{(1)} + c_{(2)} t + c_{(3)} t^2 \in \mathcal{U}$ is satisfied for the whole time horizon.

\paragraph{Feedback Control}
\label{para:feedback}

We can also apply a feedback control law $u(t) = u_{ref}(t) + K (x(t) - x_{ref}(t))$ with a fixed feedback matrix $K \in \R^{m \times n}$, where both piecewise constant or polynomial control inputs can be used for the reference control input $u_{ref}(t)$ corresponding to the reference trajectory $x_{ref}(t)$. For the safety shield, we then compute the reachable set for the extended system
\begin{equation*}
	\begin{bmatrix} \dot x(t) \\ \dot u_{ref}(t) \\ \dot x_{ref}(t) \end{bmatrix} = \begin{bmatrix} f\big( x(t), u_{ref}(t) + K(x(t) - x_{ref}(t)),w(t) \big) \\ \mathbf{0} \\ f\big(x_{ref}(t), u_{ref}(t), w(t)\big) \end{bmatrix}
\end{equation*}
using the initial set $x_0 \times \mathcal{U} \times x_0$. In the optimization problem \eqref{eq:optimize} we then determine the optimal parameter for the reference control inputs $u_{ref}(t)$, where we add the constraint $u_{ref} + K (x(t) - x_{ref}(t)) \in \mathcal{U}$ to satisfy the input constraints. 

\begin{figure}[!tb]
	\centering
	\setlength{\belowcaptionskip}{-15pt}
	\psfragfig[width=0.99\columnwidth]{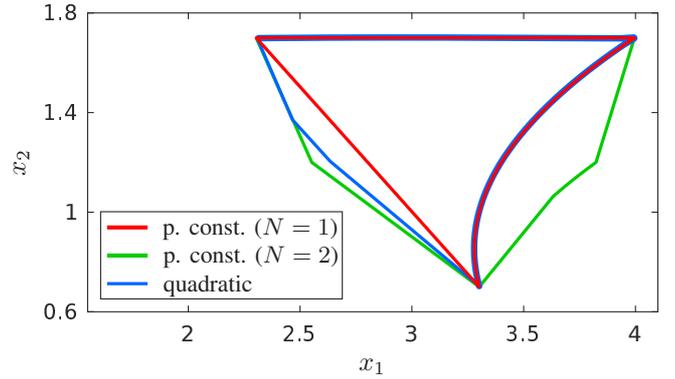}{
		\psfrag{a}[c][c]{$x_1$}
		\psfrag{b}[c][c]{\rotatebox[origin=c]{180}{$x_2$}}
		\psfrag{c}[l][c]{\small p. const. ($\pieces =1$)}
		\psfrag{e}[l][c]{\small p. const. ($\pieces =2$)}
		\psfrag{u}[l][c]{\small quadratic}}
	\caption{Final reachable set for the system in Example~\ref{ex:runningExample} for a quadratic control law and piecewise constant control laws with different numbers of segments \captionmath{\pieces}.} 
	\label{fig:controlLaws}
\end{figure}

A comparison of the different control laws presented in this section is shown in Fig.~\ref{fig:controlLaws} for the system in Example~\ref{ex:runningExample}. The results demonstrate that even for a piecewise constant control law with only $\pieces=2$ segments we already obtain a larger reachable set than with a quadratic control law, which increases our chances to find a safe control input. While piecewise constant control laws therefore seem to be preferable, their rapidly changing values often negatively impact comfort or durability for many systems, which can be avoided with polynomial control laws.

For all control strategies we apply the following control scheme: We plan for a time horizon of $t_f$, but execute the resulting control law for only a shorter time period $t_c < t_f$ before planning a new trajectory. This increases the chances to avoid getting stuck in dead ends and ensures that we can react quickly to dynamic changes in the environment. 

\subsection{Spatial Dimensions of Mobile Robots}
\label{subsec:vehicleDim}

So far we considered the case where the safety constraints are specified directly by the system state. For collision avoidance, however, this setup is usually not sufficient since we additionally have to consider the shape and spatial dimension of mobile robots, e.g., cars, vessels, or drones, which we want to control safely. 
While for many other approaches this poses a huge problem, incorporating spatial dimensions of the robot into our safety shield is quite straightforward since we simply have to replace the reachable set with the occupancy set. Given the reachable set $\mathcal{R}(t)$ that typically describes all possible positions of a reference point on the robot as well as all possible robot orientations, the occupancy set is defined as
\begin{equation}
	\mathcal{O}(t) = \big\{ o(x,d) ~\big|~ x \in \mathcal{R}(t),~d \in \mathcal{D} \big\},
	\label{eq:occupancy}
\end{equation}
where the function $o: \Rn \times \R^\delta \to \R^\gamma$ describes how the space occupied by the robot is computed from the system state and the set $\mathcal{D}$ specifying the spatial dimension of the robot.
\begin{example}
	Let us consider a car where the states $x_{(1)}$ and $x_{(2)}$ describe the x- and y-position of its center, and state $x_{(3)}$ describes the orientation of the car. Then the function $o(x,d)$ that defines the space occupied by the car is given as
	\begin{equation}
		o(x,d) = \begin{bmatrix} x_{(1)} + \cos(x_{(3)}) \, d_{(1)} - \sin(x_{(3)}) \, d_{(2)} \\ x_{(2)} + \sin(x_{(3)}) \, d_{(1)} + \cos(x_{(3)}) \, d_{(2)} \end{bmatrix},
		\label{eq:outputCar}
	\end{equation}
	where the shape of the car is for simplicity enclosed by a rectangle, so that $d \in \mathcal{D} = [-l/2,l/2] \times [-w/2,w/2]$ with $l$ and $w$ denoting the length and width of the car.
\end{example}
To compute the occupancy set \eqref{eq:occupancy} from the reachable set $\mathcal{R}(t)$ and the set $\mathcal{D}$ using polynomial zonotopes, we suggest two approaches:
\begin{enumerate}
	\item We can compute a Taylor series expansion enclosure of the function $o(x,d)$ and evaluate it in a set-based way to obtain 
	the occupancy set $\mathcal{O}(t)$ \cite[Sec.~4.4]{Kochdumper2020b}.
	\item Since polynomial zonotopes can be converted to Taylor models \cite[Prop.~4]{Kochdumper2019} we can apply Taylor model arithmetic \cite{Makino2003} to evaluate \eqref{eq:occupancy} and then convert the resulting set back to a polynomial zonotope.  
\end{enumerate}

The resulting safety constraints that we obtain from the intersections between the occupancy set $\mathcal{O}(t)$ and obstacles have to hold for all values $d \in \mathcal{D}$. To ensure this, we could represent the set $\mathcal{D}$ with independent generators before computing $\mathcal{O}(t)$, similarly as we did for the set of measurement errors in Sec.~\ref{sec:main}. However, since the set $\mathcal{D}$ is in general much larger than the set of measurement errors $\mathcal{V}$, this would often yield very conservative results. A better approach is to project out all factors that correspond to the set $\mathcal{D}$ using Fourier-Motzkin elimination \cite[Chapter~4.4]{Dantzig2016}. Let us demonstrate this by an example: 
\begin{example}
	We consider the constraint 
	\begin{equation}
		\forall \alpha_3 \in [-1,1]: ~~ \alpha_1 + \alpha_2 + \alpha_3 + \alpha_1^2 \alpha_3 \leq 1.5,
		\label{eq:fourierMotzkin1}
	\end{equation}
	from which we want to eliminate $\alpha_3$. The first step of Fourier-Motzkin elimination is to solve all constraints for $\alpha_3$, which yields
	\begin{equation}
		\alpha_3 \leq \frac{1.5 - \alpha_1 - \alpha_2}{1 + \alpha_1^2}, \quad \alpha_3 \leq 1, \quad \alpha_3 \geq \shortminus 1.
		\label{eq:fourierMotzkin}
	\end{equation}
	Next, we have to form all combinations of the constraints in \eqref{eq:fourierMotzkin} that result in a non-empty solution, yielding the constraints
	\begin{equation*}
		\begin{split}
			\frac{1.5 - \alpha_1 - \alpha_2}{1 + \alpha_1^2} \geq \shortminus 1 \quad &\Rightarrow \quad \shortminus \alpha_1^2 + \alpha_1 + \alpha_2 \leq 2.5 \\
			\frac{1.5 - \alpha_1 - \alpha_2}{1 + \alpha_1^2} \leq 1 \quad &\Rightarrow \quad \alpha_1^2 + \alpha_1 + \alpha_2 \geq 0.5,
		\end{split}
	\end{equation*}
	which represent an equivalent formulation of \eqref{eq:fourierMotzkin1}.
\end{example}
Since Fourier-Motzkin elimination requires that the constraints are solvable for the variable that is eliminated, all terms that violate this condition have to be removed first by applying a zonotope enclosure \cite[Prop.~5]{Kochdumper2019}.

\subsection{Mixed-Integer Linear Program Formulation}
\label{subsec:mixedInteger}

For some systems, solving the nonlinear mixed-integer optimization problem \eqref{eq:optimize} might be computationally too expensive, especially when we have to evaluate the safety shield in real-time for online application. Therefore, we now discuss how to obtain a feasible and close to optimal solution using mixed-integer linear programming, which is significantly faster. To achieve this, we enclose the polynomial zonotopes that represent the reachable set with zonotopes using \cite[Prop.~5]{Kochdumper2019}. Since zonotopes are linear in the factors $\alpha$, the feasible region for $\alpha$ calculated using Thm.~\ref{prop:intersect} is then given as a union of polytopes $\bigcup_l \langle A_l, b_l\rangle_P$ instead of a union of polynomial level sets. Consequently, if we additionally minimize the $L^1$-norm instead of the $L^2$-norm, we can simplify the optimization problem \eqref{eq:optimize} to
\begin{equation*}
	\min_{\alpha \in [\shortminus \mathbf{1},\mathbf{1}]} \| \alpha - \alpha_{\subRL} \|_1  ~~ \text{s.t.} ~~ \alpha \in \bigcap_{\ind{}=1}^{\numInt} \bigcup_{l=1}^{\polyCons_\ind{}} \langle A_{\ind{}l}, b_{\ind{}l} \rangle_P,
\end{equation*}	
which can be formulated as a mixed-integer linear program using Balas' Theorem \cite{Balas1998}:
\begin{equation}
	\min_{\alpha \in [\shortminus \mathbf{1},\mathbf{1}]} \| \alpha - \alpha_{\subRL} \|_1
	\label{eq:linProg}
\end{equation}
subject to
\allowdisplaybreaks
\begin{align*}
	& A_{\ind{}l} \, \widehat{\alpha}_{\ind{}l} \leq \lambda_{\ind{}l} \, b_{\ind{}l}, ~ \shortminus \mathbf{1} \, \lambda_{\ind{}l} \leq \widehat{\alpha}_{\ind{}l} \leq \mathbf{1} \, \lambda_{\ind{}l}, \\
	& \, \lambda_{\ind{}l} \in \{0,1\}, ~ \alpha = \sum_{l=1}^{\polyCons_\ind{}} \widehat{\alpha}_{\ind{}l}, ~\sum_{l=1}^{\polyCons_\ind{}} \lambda_{\ind{}l} = 1,
\end{align*}
for $\ind{} = 1,\dots,\numInt$ and $l = 1,\dots,\polyCons_\ind{}$. The structure of this optimization problem is very similar to \eqref{eq:optimize}, except that we introduced the new variables $\widehat{\alpha}_{\ind{}l} = \alpha_{\ind{}l} \lambda_{\ind{}l}$ to avoid the bilinear terms and obtain a linear program. Due to the over-approximation of all nonlinear terms of the polynomial zonotope by the zonotope enclosure, it holds that every feasible solution for \eqref{eq:linProg} is a feasible solution to the original problem \eqref{eq:optimize}, but some values that are feasible for \eqref{eq:optimize} will not be feasible for \eqref{eq:linProg}. 
Note that if the system dynamics \eqref{eq:system} is linear, we directly obtain a mixed-integer linear program in the form of \eqref{eq:linProg}. Moreover, we can always first check if the desired value $\alpha_{\subRL}$ satisfies the original nonlinear constraints and only perform the simplification to a mixed-integer linear program if it does not. A mixed-integer quadratic program can be obtained in a similar way as the mixed-integer linear program by enclosing all generators that belong to higher-order polynomials by a zonotope. Finally, since mixed-integer programming can be highly parallelized, the computation time for optimization can always be reduced by using a more powerful machine with more cores. 

\subsection{Constraint Grouping}
\label{subsec:constraintGrouping}

Since the time step size for reachability analysis is usually relatively small, it often happens that many reachable sets for consecutive time intervals intersect the same obstacle, resulting in a lot of very similar constraints. 
We can reduce the computation time by grouping similar constraints together, as we demonstrate with the following example:
\begin{example}
	The two constraints 
	\begin{equation*}
		\begin{split}
			& 1.1 \, \alpha_1 + 0.7 \, \alpha_1 \alpha_2 \leq 0.3 \\
			& 1.3 \, \alpha_1 + 0.5 \, \alpha_1 \alpha_2 \leq 0.3
		\end{split}
	\end{equation*}
	on $\alpha_1,\alpha_2 \in [\shortminus 1,1]$ can be grouped to the single constraint
	\begin{equation*}
		\forall \epsilon_1 \in [1.1, 1.3], \, \forall \epsilon_2 \in [0.5, 0.7]: ~ \epsilon_1 \, \alpha_1 + \epsilon_2 \, \alpha_1 \alpha_2 \leq 0.3.
	\end{equation*}
	To eliminate the new variables $\epsilon_1$ and $\epsilon_2$ we represent their domains as a summation of the center with a zero-centered uncertainty as $\epsilon_1 \in 1.2 + \widetilde{\epsilon}_1$, $\epsilon_2 \in 0.6 +  \widetilde{\epsilon}_2$ with $\widetilde{\epsilon}_1,\widetilde{\epsilon}_2 \in [\shortminus 0.1,0.1]$, which finally yields
	\begin{equation*}
		1.2 \, \alpha_1 + 0.6 \, \alpha_1 \alpha_2 \leq \hspace{-10pt} \min_{\substack{ \\ \alpha_1,\alpha_2 \in [\shortminus 1,1] \\ \widetilde{\epsilon}_1,\widetilde{\epsilon}_2 \in [\shortminus 0.1,0.1] }} \hspace{-10pt} 0.3 - \epsilon_1 \, \alpha_1 - \epsilon_2 \, \alpha_1 \alpha_2,
	\end{equation*}
	where a lower bound for the optimal value of the minimization problem can be computed using interval arithmetic \cite{Jaulin2006}.
\end{example}	

In addition to the number of constraints, constraints grouping also decreases the number of integer variables for the optimization in \eqref{eq:optimize}, which reduces computation time. Since integer variables are required only if the safe region for the agent is non-convex, another strategy to accelerate the optimization is to replace non-convex safe regions by the largest convex subset \cite{Deits2015}.

\subsection{Reachable Set Pre-Computation}
\label{subsec:reachSetPrecomputation}

In order to reduce the computation time for our safety shield, we can pre-compute the reachable set starting from an initial set $\mathcal{X}_0$ offline, and then apply the reachable subset approach \cite{Kochdumper2020c} to efficiently extract the reachable set for the current state $x_0 \in \mathcal{X}_0$ during online execution. Since for nonlinear systems the accuracy of the reachable set enclosure depends on the size of the initial set, we cannot make $\mathcal{X}_0$ too large but instead have to divide the relevant state space into sets of suitable size. The number of required sets for such a division grows exponentially with the system dimension, so that this approach is not suited for high-dimensional systems. However, for many systems the differential equation $\dot x(t) = f(x(t),u(t),w(t))$ describing the system dynamics is invariant with respect to transformations of certain states \cite[Sec.~4.1]{Bak2015b}. For example, the dynamics of a car are invariant with respect to translations of the car's position and with respect to rotations of the car's orientation. In this case only the state space for the states that are not invariant has to be divided since we can always apply a suitable state space transformation to set the invariant states to 0.

\section{Experimental Evaluation}
\label{sec:exp}

We now demonstrate the performance of our safety shield on several benchmark systems, where each benchmark highlights different properties of our approach.
If not explicitly stated otherwise, all computations are carried out in Python on a 2.9GHz quad-core i7 processor with 32GB memory. We use the CORA toolbox \cite{Althoff2015a} to pre-compute reachable sets, proximal policy optimization \cite{Schulman2017} for reinforcement learning, Gurobi to solve the mixed-integer linear and quadratic programs, and CasADi together with the BONMIN solver to solve mixed-integer nonlinear programs\footnote{\url{https://www.gurobi.com/} and \url{https://web.casadi.org/}}. Benchmark parameters as well as the applied extensions from Sec.~\ref{sec:impr} are listed in Tab.~\ref{tab:benchmarks}. We published our implementation on CodeOcean\footnote{\url{https://codeocean.com/capsule/9949621/tree/v1}} and created a video showing our results\footnote{\url{https://youtu.be/6ISKxO4DDWA}}.

\begin{table}[tb]
	\caption{States \captionmath{n}, control inputs \captionmath{m}, planning horizon \captionmath{t_f}, number of pre-computed reachable sets, and extensions applied for each benchmark.}
	\vspace{-5pt}
	\label{tab:benchmarks}
	\begin{center}
		\renewcommand{\arraystretch}{1.2}
		\begin{tabular}{ l c c c c c}
			\toprule[0.5pt]
			\hspace{-4pt}\textbf{Benchmark}  & $\mathbf{n}$ & \hspace{-5pt}$\mathbf{m}$\hspace{-5pt} & $\mathbf{t_f}$ & \hspace{-5pt}\textbf{Sets}\hspace{-5pt} & \textbf{Extensions} \\
			\midrule[0.25pt]
			\hspace{-4pt}F1tenth car    &  5 & 2 & 2\,\si{\second} & 5 & \ref{para:piecewise}, \ref{subsec:vehicleDim}, \ref{subsec:mixedInteger}, \ref{subsec:reachSetPrecomputation}                \\
			\hspace{-4pt}Auto. driving & 4 & 2 & 0.8\,\si{\second} & 60 & \hspace{-3pt}\ref{para:feedback},\,\ref{subsec:vehicleDim},\,\ref{subsec:mixedInteger},\,\ref{subsec:constraintGrouping},\,\ref{subsec:reachSetPrecomputation}\hspace{-3pt}         \\
			\hspace{-4pt}Quadrotor 2D & 6 & 2 & 0.5\,\si{\second} & \hspace{-5pt}256\hspace{-5pt} & \ref{subsec:reachSetPrecomputation} \\
			\hspace{-4pt}Quadrotor 3D & \hspace{-5pt}10\hspace{-5pt} & 3 & 3\,\si{\second} & 1 & \ref{subsec:vehicleDim}, \ref{subsec:mixedInteger}, \ref{subsec:reachSetPrecomputation} \\
			\bottomrule[0.5pt]
		\end{tabular}
	\end{center}
	\vspace{-15pt}
\end{table}

\subsection{F1tenth Racecar}

To demonstrate that our safety shield is fast and robust enough to be applied to a real system, we conduct experiments on an F1tenth racecar \cite{Kelly2020}, whose dynamics are described by a kinematic single-track model. Moreover, the car contains a low-level PI controller with gains $k_P = 8$ and $k_I = 1$ that takes as input the desired velocity and realizes the required acceleration. Overall, this results in the model
\begin{equation}
	\begin{split}
		& \dot s_x = \cos(\psi) \, v, ~ \dot s_y = \sin(\psi) \, v, ~ \dot \psi = u_2 + w_2, \\
		& \dot v = k_P (u_1 - v) + k_I \, e_I + w_1, ~ \dot e_I = u_1 - v,
	\end{split}	
	\label{eq:f1tenth}
\end{equation}
where the system state consists of the x- and y-position of the center $s_x,s_y$, the velocity $v$, the orientation $\psi$, and the integrated error of the PI controller $e_I$. The control inputs are the desired velocity $u_1$ and the steering angle $u_2$, which are bounded by the set $\mathcal{U} = [0, 0.5] \si{\metre \per \second} \times [-0.3,0.3] \si{\radian}$. To ensure that the model \eqref{eq:f1tenth} encloses all possible behaviors of the real system, we performed conformance checking using the AROC toolbox \cite{Kochdumper2021b} to determine the process noise as well as the measurement error from data traces we recorded from the real car, which results in the sets 
\begin{align*}
	& \mathcal{W} = [-0.007, 0.0035] \si{\metre \per \square \second} \times [-0.0104, 0.0132] \si{\radian \per \second} \\
	& \mathcal{V} = [-0.0584, 0.0446] \si{\metre \per \second} \times [-0.0438, 0.0466] \si{\metre \per \second} \times \\
	& \quad \quad [-0.0933, 0.0561] \si{\metre \per \square \second} \times [-0.0446, 0.0593] \si{\radian \per \second} \times \\
	& \quad \quad [-0.0005, 0] \si{\metre \per \square \second}.
\end{align*} 
To incorporate the size of the car, we use the output function in \eqref{eq:outputCar} with length $0.51 \, \si{\metre}$ and width $0.31 \,\si{\metre}$. 

For control we use a piecewise constant control law with $\pieces = 2$ segments and a planning horizon of $t_f = 2 \, \si{\second}$, and we replan as soon as the previous computation is finished. Moreover, we simplify the optimization problem for action projection to a mixed-integer quadratic program, which on average took $0.14 \, \si{\second}$ to solve during our experiments. 
The car uses a 1.9GHz six-core ARMv8 processor with 7.6GB memory and is equipped with a LiDAR sensor. To obtain the unsafe sets $\mathcal{F}_i$, we enclose all points measured by the LiDAR by a union of polytopes.
Moreover, while the velocity and the integrated error can be directly obtained from the car's internal sensors, we use a particle filter \cite{Thrun2001} to determine the position and orientation of the car in the environment from LiDAR measurements. For our experiments, we then applied reinforcement learning to train an agent on four environment maps that were similar to but slightly different from the map we used for the experiments on the real F1tenth car. In addition to the system state, we used the LiDAR measurements and the position of the goal set as observations for the agent, and we did not use the safety shield during training. 

As shown in Fig.~\ref{fig:f1tenth}, without the safety shield, the trained agent is unsafe since the car crashes into the obstacle. With our safety shield, however, the car avoids the obstacle and safely reaches the goal set. This not only demonstrates that our safety shield successfully works on a real system, but also that the modifications to the control inputs suggested by the reinforcement learning policy are small enough for the agent to still fulfill its objective.

\begin{figure}[!tb]
	\centering
	\setlength{\belowcaptionskip}{-15pt}
	\psfragfig[width=0.99\columnwidth]{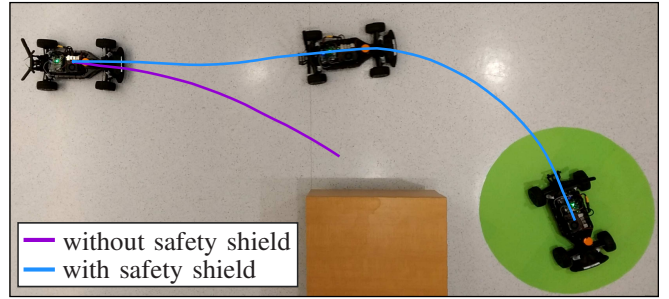}{
		\psfrag{a}[l][c]{without safety shield}
		\psfrag{c}[l][c]{with safety shield}}
	\caption{Trajectories driven by the F1tenth racecar with and without the safety shield, where the green area is the goal set and the orange area is the obstacle.
	} 
	\label{fig:f1tenth}
\end{figure}

\begin{figure*}[!tb]
	\centering
	\setlength{\belowcaptionskip}{-13pt}
	\psfragfig[width=1.99\columnwidth]{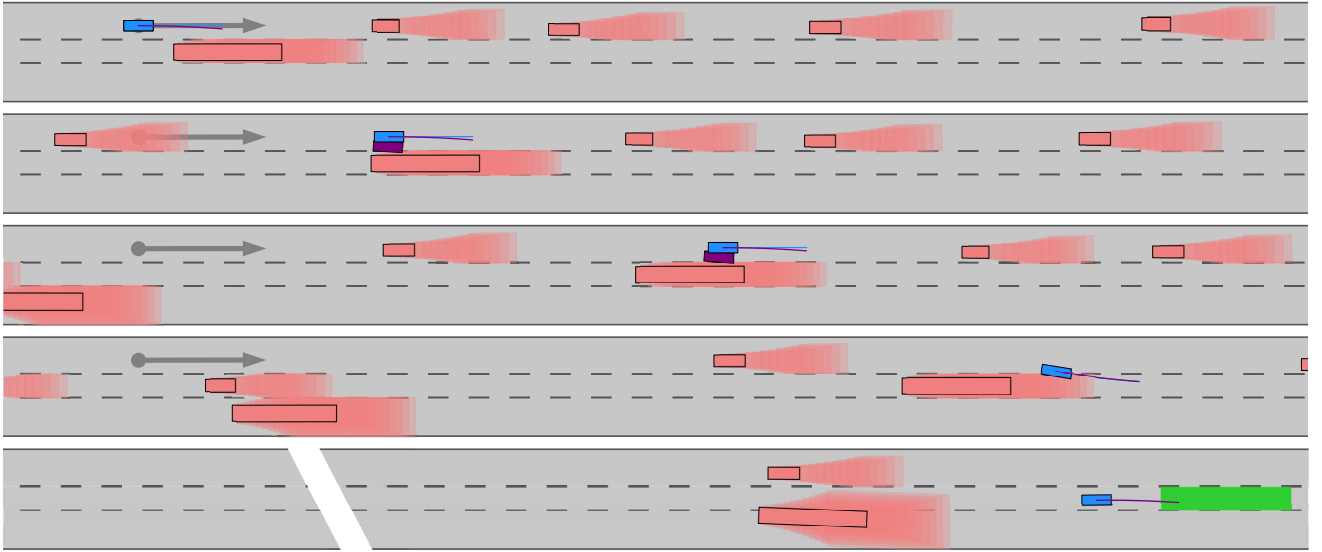}{
		\psfrag{a}[c][c]{without safety shield}
		\psfrag{c}[c][c]{with safety shield}}
	\caption{Results for the CommonRoad scenario \textit{DEU\_LocationALower-33\_16\_T-1} visualized at times \captionmath{0\,\si{\second}}, \captionmath{1.2\,\si{\second}}, \captionmath{2.8\,\si{\second}}, \captionmath{4.4\,\si{\second}}, and \captionmath{9.2\,\si{\second}} (from top to bottom), where the agent without safety shield is depicted in purple, the agent with safety shield is depicted in blue, the dynamic obstacles are depicted in red, and the goal set is depicted in green.} 
	\label{fig:commonRoad}
\end{figure*}

\subsection{Autonomous Driving}

In order to show that our safety shield can handle very complex reach-avoid problems that include dynamic obstacles, we consider the motion planning benchmarks for autonomous cars provided by the CommonRoad database \cite{Althoff2017a}. As system dynamics we use the kinematic single track model from \cite[Sec.~VII]{Schurmann2021} with the same input set $\mathcal{U}$ and set of process noise $\mathcal{W}$ as in \cite[Sec.~VII]{Schurmann2021}. This model is very similar to the model in \eqref{eq:f1tenth}, with the only difference that the acceleration instead of the desired velocity is used as a control input. The car we consider is a BMW 320i, which has a length of $4.51 \,\si{\metre}$ and a width of $1.61\,\si{\metre}$. To guarantee safety even though the intentions of the other cars are unclear, we use the tool SPOT \cite{Koschi2017a} to compute all possible occupancies of the other traffic participants that apply to traffic rules using set-based prediction.  

\begin{table}[!tb]
	\caption{Results for the evaluation of our safety shield on 2000 CommonRoad traffic scenarios.}
	\vspace{-7pt}
	\label{tab:commonRoad}
	\begin{center}
		\renewcommand{\arraystretch}{1.2}
		\begin{tabular}{ l c c c}
			\toprule[0.5pt]
			\textbf{Agent}  & \textbf{Collisions} & \textbf{Goal Reached} & \textbf{Time} \\
			\midrule[0.25pt]
			without safety shield    & 10 (0.5\%)  & 1907 (95.35\%) & - \\
			with safety shield      & 0 (0\%)  & 1925 (96.25\%) & $0.043\,\si{\second}$ \\
			constraint grouping & 0 (0\%)  & 1924 (96.20\%) & $0.035\,\si{\second}$ \\
			\bottomrule[0.5pt]
		\end{tabular}
	\end{center}
	\vspace{-15pt}
\end{table}

To counteract the large process noise for this benchmark, we use a feedback controller $u(t) = u_{ref} + K(x(t)-x_{ref}(t))$ for the safety shield, where the reference input $u_{ref}$ is piecewise constant with $\pieces=2$ segments. The feedback matrix $K \in \R^{m \times n}$ is determined by applying an LQR control approach with state weighting matrix $Q = I_4$ and input weighting matrix $R = I_2$ to the linearized system. Moreover, we use a planning horizon of $t_f = 0.8\,\si{\second}$ and replan after $t_c = 0.4\,\si{\second}$. We apply reinforcement learning to train an agent that aims to safely control the car, where we do not use the safety shield during training. The observations for the agent are selected from \cite[Tab.~II]{Wang2021}. In particular, we use the state of the ego vehicle, the distances of the ego vehicle to road/lane boundaries as well as to the goal set, and the states of surrounding vehicles. 

The effect of the safety shield is highlighted by the results for 2000 traffic scenarios shown in Tab.~\ref{tab:commonRoad}: While the original agent collides with other traffic participants in 10 scenarios, our safety shield successfully prevents all collisions. Moreover, applying the safety shield does not lead to a reduced goal-reaching percentage, but instead even increases the number of scenarios for which the goal set is reached. 
Tab.~\ref{tab:commonRoad} also demonstrates the effect of constraint grouping (see Sec.~\ref{subsec:constraintGrouping}), which reduces the average computation time for solving the optimization problem, but slightly decreases the goal reaching percentage due to the increased conservatism.
In Fig.~\ref{fig:commonRoad} the results for one specific traffic scenario are visualized. There, the agent without the safety shield changes the lane too early and collides with the adjacent truck, whereas the agent with the safety shield changes the lane just in time and finally reaches the goal set in the end.

\begin{figure}[!tb]
\centering
\setlength{\belowcaptionskip}{-13pt}
\psfragfig[width=0.99\columnwidth]{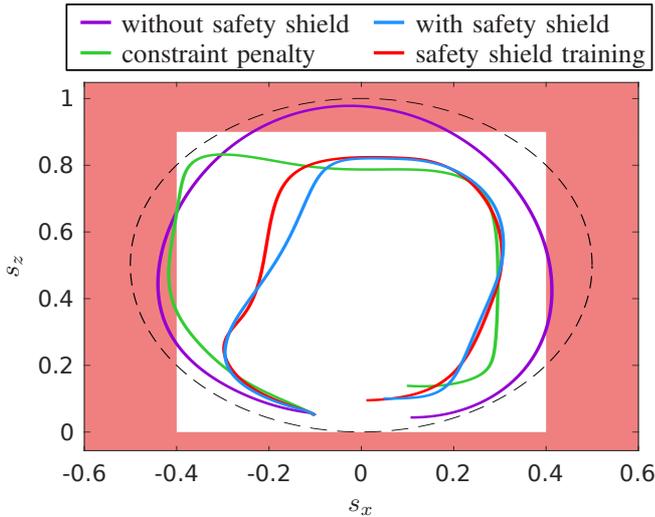}{
	\psfrag{a}[c][c]{$s_x$}
	\psfrag{b}[c][c]{\rotatebox[origin=c]{180}{$s_z$}}
	\psfrag{c}[l][c]{without safety shield}
	\psfrag{u}[l][c]{with safety shield}
	\psfrag{x}[l][c]{constraint penalty}
	\psfrag{e}[l][c]{safety shield training}  
}
\caption{Trajectories for the 2D quadrotor benchmark featuring the baseline agent with and without safety shield, the constraint-penalty agent, and the agent that is trained with the safety shield. The trajectory that should be tracked is visualized by the dashed black line and the unsafe regions are depicted in red.} 
\label{fig:quadcopter2D}
\end{figure}

\subsection{Quadrotor 2D}

Next, we compare our safety shield with a safe reinforcement learning approach that modifies the optimization criterion. In particular, we incorporate the safety specification as a violation penalty in the reward function. For this, we consider a benchmark problem from the safe-control-gym \cite{Yuan2022} featuring a trajectory tracking task for a two-dimensional quadrotor. As shown in Fig.~\ref{fig:quadcopter2D}, the trajectory that should be tracked is partially located inside an unsafe region, so that there exists a conflict between tracking performance and safety constraint satisfaction. The dynamics of the quadrotor are according to \cite[Eq.~(3)]{Yuan2022} given as
\begin{equation*}
\begin{split}
	& \ddot s_x = \sin(\psi) \, (u_1 + u_2) / m + w_1 \\
	& \ddot s_z = \cos(\psi) \, (u_1 + u_2) / m - g + w_2 \\
	& \ddot \psi = (u_2 - u_1) \, a / \big(\sqrt{2} \, I_{yy}\big) + w_3,
\end{split}
\end{equation*}
where $m = 0.027\,\si{\kilogram}$ is the mass, $g = 9.81\,\si{\meter \per \square \second}$ is the gravitational acceleration, $a = 0.0397\,\si{\meter}$ is distance from each motor pair to the center of mass of the quadrotor, and $I_{yy} = 1.4\cdot 10^{-5} \,\si{\kilogram \per \square \meter}$ is the moment of inertia. The system state consists of the x- and z-positions $s_x, s_z$ as well as the pitch angle $\psi$ of the quadrotor together with the corresponding velocities. To decouple forward thrust and tilting torque, the input set for the control inputs $u_1$ and $u_2$ that represent the thrusts generated by the two rotors is restricted to
\begin{equation*}
\mathcal{U} = \bigg\langle \begin{bmatrix} 0.1323 \, \si{\newton} \\ 0.1323 \, \si{\newton} \end{bmatrix}, \begin{bmatrix} 0.0125 \, \si{\newton} & 0.0015 \, \si{\newton} \\ 0.0125 \, \si{\newton} & -0.0015 \, \si{\newton} \end{bmatrix} \bigg \rangle_Z.
\end{equation*}
The process noise $w_1,w_2,w_3$ is bounded by the set $\mathcal{W} = 0.01 \cdot [\shortminus \mathbf{1},\mathbf{1}]$.

For the safety shield we use a constant control input with a planning horizon of $t_f = 0.5\,\si{\second}$, where we replan after $t_c = 0.02\,\si{\second}$. To perform action projection, we solve the original nonlinear optimization problem, which takes $0.004\,\si{\second}$ on average during our experiments. The main reason for the fast computation time is that the safe region for the quadrotor is convex, which results in an optimization problem without any integer variables. We train three different agents: A baseline agent that should track the trajectory and gets no information about the constraints, a constraint-penalty agent where the reward is extended with a penalty for constraint violation, and a safe agent that is trained with the safety shield. As shown in Fig.~\ref{fig:training}, the safe and baseline agents converge after 400 000 training steps while the agent with constraint penalty needs 2 million training steps to converge. Moreover, only the safe agent never violates any constraints during training, and could therefore also be used for training directly on the real physical system. 

The results for deploying the different trained agents are shown in Fig.~\ref{fig:quadcopter2D}. As expected, the baseline agent without the safety shield violates the safety constraints since they were not considered during training. Also, the constraint-penalty agent violates the constraints, which demonstrates that it is not sufficient to incorporate the safety constraints into the training process. Only the two agents that apply our safety shield stay inside the safe region for all times, where the agent that uses the safety shield during training achieves a smoother trajectory compared to the baseline agent. 

\begin{figure}[!tb]
\centering
\setlength{\belowcaptionskip}{-13pt}
\psfragfig[width=0.99\columnwidth]{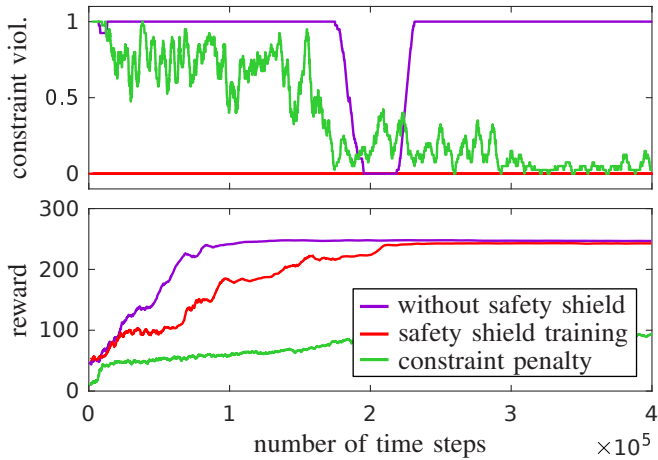}{
	\psfrag{a}[c][c]{number of time steps}
	\psfrag{b}[c][c]{\rotatebox[origin=c]{180}{constraint viol.}}
	\psfrag{d}[c][c]{\rotatebox[origin=c]{180}{reward}}
	\psfrag{c}[l][c]{without safety shield}
	\psfrag{e}[l][c]{safety shield training}
	\psfrag{u}[l][c]{constraint penalty}
}
\caption{Episode rewards and constraint violations for the 2D quadrotor benchmark observed during training without safety shield, with safety shield, and with constraint penalty.} 
\label{fig:training}
\end{figure}

\subsection{Quadrotor 3D}

To compare our safety shield with reachability-based trajectory safeguard \cite{Shao2021}, we consider the three-dimensional quadrotor benchmark from \cite[Sec.~V.B]{Shao2021}. Reachability-based trajectory safeguard \cite{Shao2021} applies the safety shield to a simplified trajectory-generating model and the resulting trajectory is then tracked by a low-level controller that uses the original nonlinear system model. For the quadrotor, the trajectory-generating model for each of the three spatial directions $i \in \{1,2,3\}$ is
\begin{equation*}
\begin{split}
	& \dot x_i = v_{i} \hspace{-1pt}+\hspace{-1pt} a_{i}\,t \hspace{-1pt}-\hspace{-1pt} \big( 2 a_{i} \hspace{-1pt}+\hspace{-1pt} 3 (v_{i} \hspace{-1pt}-\hspace{-1pt} u_i) \big) t^2 \hspace{-1pt}+\hspace{-1pt} \big( a_{i} \hspace{-1pt}+\hspace{-1pt} 2 (v_{i} \hspace{-1pt}-\hspace{-1pt} u_i) \big) t^3 \\
	& \quad \quad \quad \quad \quad \dot v_{i} = 0, \quad \quad \dot a_{i} = 0, \quad \quad \dot t = 1,
\end{split}
\end{equation*}
where $x_i$ is the quadrotor position, $v_{i}$ and $a_{i}$ are respectively the velocity and the acceleration at the beginning of the trajectory, and $t$ is time. The input $u_i$ to the system is the peak velocity 
reached at time $t = 1.5\,\si{\second}$, which is bounded by the set $\mathcal{U} = \{u = [u_1~u_2~u_3]^T~|~ \| u \|_2 \leq 5 \,\si{\metre \per \second}\}$. To apply our safety shield, we tightly inner-approximate the set $\mathcal{U}$ with a zonotope using the method described in \cite[Sec.~IV]{Gassmann2020}. 
A similar trajectory-generating model is used to decelerate the quadrotor from the peak velocity back to velocity 0, so that the overall planning horizon is $t_f = 3\,\si{\second}$. 
We consider the same control task as in \cite[Sec.~V.B]{Shao2021}, which is to safely navigate the quadrotor through a $100\,\si{\metre}$ long tunnel with randomly generated box obstacles. For our experiments, we deployed the same trained reinforcement learning agent as used in \cite{Shao2021} on 100 tunnels with different obstacles and compared the conservatism of the two safety shields in terms of the required control input correction $\| u - u_{\subRL} \|_2$ at each intervention of the safety shield. 
While both safety shields had to intervene for 5078 out of 5760 time steps, the average control input modification for our approach is with $1.13\,\si{\metre \per \second}$ smaller than the average modification $1.22\,\si{\metre \per \second}$ for the safety shield from \cite{Shao2021}, which increases the chances that the agent can successfully complete its task. 
´
\section{Discussion}
\label{sec:disc}

Finally, let us discuss some properties of our safety shield.

\subsection{Safety Guarantees for Infinite Time}

Our basic safety shield approach can guarantee safety only for the finite time horizon $t_f$. To obtain safety guarantees for an infinite time horizon, one can either combine our safety shield with a fail-safe planner \cite{Pek2018b} that takes over when the safety shield cannot determine a safe trajectory anymore, or one can modify the safety shield in such a way that the system always stops in a safe final state at the end of the planning horizon \cite{Shao2021}.

\subsection{Computational Complexity}

The two main steps required for our safety shield are computing the reachable set and solving the mixed-integer optimization problem \eqref{eq:optimize} for action projection. The complexity of the conservative polynomialization algorithm for reachability analysis is $\mathcal{O}(n^5)$ with respect to the system dimension according to \cite[Sec.~4.1.4]{Kochdumper2022}. However, for many benchmarks one can apply the pre-computation discussed in Sec.~\ref{subsec:reachSetPrecomputation} to avoid computing reachable sets online. Solving a mixed-integer optimization problem is in general NP-hard \cite{Papadimitriou1981}. But, as we demonstrated with the numerical experiments in Sec.~\ref{sec:exp}, by applying the simplification to a mixed-integer linear program in Sec.~\ref{subsec:mixedInteger} and/or constraint grouping in Sec.~\ref{subsec:constraintGrouping} we can solve this optimization efficiently.

\subsection{Safe Computation Time Consideration}

As demonstrated by the experiments in Sec.~\ref{sec:exp}, even with all the speed-ups discussed in Sec.~\ref{sec:impr}, the calculations required for our safety shield still need a certain amount of computation time that, depending on the system, might be too long to simply be neglected. Therefore, in order to consider the required computation time in a formally correct manner, we can apply the following well-known procedure \cite{Schuermann2018a}: We allocate a certain computation time $t_{comp}$ for the calculations and use reachability analysis to predict the reachable states for the allocated computation time. By using this set as the initial set for our safety shield, we can guarantee safety even though the required calculations are not instantaneous. If the computation does not finish in the allocated computation time, we either stick to the safe solution from the previous time step or apply a failsafe maneuver. 

\subsection{Conservatism of the Safety Shield}

Due to over-approximation errors, our safety shield might not be able to always find a feasible solution if one exists. In particular, there are four sources of conservatism:
\begin{itemize}
\item Since the exact reachable set cannot be computed for general nonlinear systems, we compute a tight enclosure instead (e.g., we aim to minimize the Hausdorff distance between the enclosure and the exact set).
\item Due to dependency preservation, the abstraction error for reachability analysis is computed on the reachable set for the whole input set rather than the smaller reachable set for a specific control input, which results in additional conservatism.
\item For bloating the obstacles by the set of uncertainties defined by the independent generators, we use an over-approximative Minkowski sum in Thm.~\ref{prop:intersect} that simply pushes the obstacle halfspaces outward. 
\item Since we choose a certain type of control law in advance, we restrict the space of possible control inputs. 
\end{itemize}
However, all of these over-approximation errors can be made arbitrarily small: The over-approximation for reachability converges to zero if the time step size is reduced and/or the reachable set is split, which also eliminates the error introduced by dependency preservation. Moreover, the approximative Minkowski sum in Thm.~\ref{prop:intersect} can be replaced by the exact one and every control law can be approximated arbitrary close by a piecewise constant control law with an infinite number of piecewise constant segments. 

\subsection{Parameter Tuning}

Since the settings for reachability analysis can be tuned automatically \cite{Wetzlinger2021,Wetzlinger2022}, the main design parameters for our safety shield in addition to the type of control law discussed in Sec.~\ref{subsec:controlLaw} are the planning horizon $t_f$ and the replanning time $t_c$. A longer planning horizon $t_f$ often yields better control performance due to the larger lookahead, but also increases the computation time. Especially in the presence of dynamic obstacles, a small replanning time $t_c$ is desirable in order to be able to quickly react to a changing environment. However, a small $t_c$ requires the approach to be faster in order to run in real-time. Finally, the extensions discussed in Sec. \ref{subsec:mixedInteger}, \ref{subsec:constraintGrouping}, and \ref{subsec:reachSetPrecomputation} all reduce the computation time at the cost of introducing more conservatism.

\vspace{0.3cm}
\section{Conclusion}
\label{sec:conc}

We presented a novel safety shield for nonlinear continuous systems with input constraints that can be added to reinforcement learning agents in order to prevent them from applying unsafe actions. Since our safety shield uses set-based computations in the form of reachability analysis to determine which actions are safe and which are unsafe, it can guarantee robust safety despite process noise and measurement errors. Moreover, because our approach applies highly parallelized mixed-integer programming to project the action from the agent to the closest safe action, it is possible to reduce the computation time by using a more powerful machine with more cores.  
Finally, we demonstrated with several numerical examples as well as experiments on a real system that our safety shield modifies the actions
proposed by the reinforcement learning agent as little as necessary for robust safety.

\section*{Acknowledgment}
The first three authors contributed equally. We gratefully acknowledge the financial support from the project justITSELF funded by the European Research Council (ERC) under grant agreement No 817629 and the German Research Foundation through the research training group ConVeY under grant GRK 2428. In addition, this material is based upon work supported by the Air Force Office of Scientific Research and the Office of Naval Research under award numbers FA9550-19-1-0288, FA9550-21-1-0121, FA9550-23-1-0066 and N00014-22-1-2156. Any opinions, findings, and conclusions or recommendations expressed in this material are those of the authors and do not necessarily reflect the views of the United States Air Force or the United States Navy. 

\bibliographystyle{IEEEtran}
\bibliography{kochdumper,cpsGroup}

\begin{thebibliography}{10}
\providecommand{\url}[1]{#1}
\csname url@samestyle\endcsname
\providecommand{\newblock}{\relax}
\providecommand{\bibinfo}[2]{#2}
\providecommand{\BIBentrySTDinterwordspacing}{\spaceskip=0pt\relax}
\providecommand{\BIBentryALTinterwordstretchfactor}{4}
\providecommand{\BIBentryALTinterwordspacing}{\spaceskip=\fontdimen2\font plus
\BIBentryALTinterwordstretchfactor\fontdimen3\font minus
  \fontdimen4\font\relax}
\providecommand{\BIBforeignlanguage}[2]{{%
\expandafter\ifx\csname l@#1\endcsname\relax
\typeout{** WARNING: IEEEtran.bst: No hyphenation pattern has been}%
\typeout{** loaded for the language `#1'. Using the pattern for}%
\typeout{** the default language instead.}%
\else
\language=\csname l@#1\endcsname
\fi
#2}}
\providecommand{\BIBdecl}{\relax}
\BIBdecl

\bibitem{Zhao2020}
W.~Zhao, J.~P. Queralta, and T.~Westerlund, ``Sim-to-real transfer in deep
  reinforcement learning for robotics: {A} survey,'' in \emph{Proc. of the
  Symposium Series on Computational Intelligence}, 2020, pp. 737--744.

\bibitem{Kiran2021}
B.~R. Kiran, I.~Sobh, V.~Talpaert, P.~Mannion, A.~A. Al~Sallab, S.~Yogamani,
  and P.~P{\'e}rez, ``Deep reinforcement learning for autonomous driving: {A}
  survey,'' \emph{Transactions on Intelligent Transportation Systems}, vol.~23,
  no.~6, pp. 4909--4926, 2021.

\bibitem{Zhang2019}
Z.~Zhang, D.~Zhang, and R.~C. Qiu, ``Deep reinforcement learning for power
  system applications: {A}n overview,'' \emph{CSEE Journal of Power and Energy
  Systems}, vol.~6, no.~1, pp. 213--225, 2019.

\bibitem{Achiam2017}
J.~Achiam, D.~Held, A.~Tamar, and P.~Abbeel, ``Constrained policy
  optimization,'' in \emph{Proc. of the Int. Conference on Machine Learning},
  2017, pp. 22--31.

\bibitem{Yang2019}
T.-Y. Yang, J.~Rosca, K.~Narasimhan, and P.~J. Ramadge, ``Projection-based
  constrained policy optimization,'' in \emph{Proc. of the Int. Conference on
  Learning Representations}, 2019.

\bibitem{Hasanbeig2020}
M.~Hasanbeig, A.~Abate, and D.~Kroening, ``Cautious reinforcement learning with
  logical constraints,'' in \emph{Proc. of the Int. Conference on Autonomous
  Agents and Multiagent Systems}, 2020, pp. 483--491.

\bibitem{Wang2022}
X.~Wang, C.~Pillmayer, and M.~Althoff, ``Learning to obey traffic rules using
  constrained policy optimization,'' in \emph{Proc. of the Int. Conference on
  Intelligent Transportation Systems}, 2022, pp. 2415--2421.

\bibitem{Konighofer2021}
B.~K{\"o}nighofer, J.~Rudolf, A.~Palmisano, M.~Tappler, and R.~Bloem, ``Online
  shielding for stochastic systems,'' in \emph{Proc. of the NASA Formal Methods
  Symposium}, 2021, pp. 231--248.

\bibitem{Geibel2005}
P.~Geibel and F.~Wysotzki, ``Risk-sensitive reinforcement learning applied to
  control under constraints,'' \emph{Journal of Artificial Intelligence
  Research}, vol.~24, pp. 81--108, 2005.

\bibitem{Krasowski2022}
H.~Krasowski, J.~Thumm, M.~M{\"u}ller, X.~Wang, and M.~Althoff, ``Provably safe
  reinforcement learning: A theoretical and experimental comparison,''
  \emph{arXiv preprint arXiv:2205.06750}, 2022.

\bibitem{Krasowski2020}
H.~Krasowski, X.~Wang, and M.~Althoff, ``Safe reinforcement learning for
  autonomous lane changing using set-based prediction,'' in \emph{Proc. of the
  Int. Conference on Intelligent Transportation Systems}, 2020.

\bibitem{Isele2018}
D.~Isele, A.~Nakhaei, and K.~Fujimura, ``Safe reinforcement learning on
  autonomous vehicles,'' in \emph{Proc. of the Int. Conference on Intelligent
  Robots and Systems}, 2018, pp. 6162--6167.

\bibitem{Huang2022}
S.~Huang and S.~Onta{\~n}{\'o}n, ``A closer look at invalid action masking in
  policy gradient algorithms,'' in \emph{Proc. of the Int. FLAIRS Conference},
  2022.

\bibitem{Thumm2022}
J.~Thumm and M.~Althoff, ``Provably safe deep reinforcement learning for
  robotic manipulation in human environments,'' pp. 6344--6350, 2022.

\bibitem{Saunders2018}
W.~Saunders, G.~Sastry, A.~Stuhlm{\"u}ller, and O.~Evans, ``Trial without
  error: {T}owards safe reinforcement learning via human intervention,'' in
  \emph{Proc. of the Int. Conference on Autonomous Agents and MultiAgent
  Systems}, 2018, pp. 2067--2069.

\bibitem{Hunt2021}
N.~Hunt, N.~Fulton, S.~Magliacane, T.~N. Hoang, S.~Das, and A.~Solar-Lezama,
  ``Verifiably safe exploration for end-to-end reinforcement learning,'' in
  \emph{Proc. of the Int. Conference on Hybrid Systems: Computation and
  Control}, 2021, article 14.

\bibitem{Alshiekh2018}
M.~Alshiekh, R.~Bloem, R.~Ehlers, B.~K{\"o}nighofer, S.~Niekum, and U.~Topcu,
  ``Safe reinforcement learning via shielding,'' in \emph{Proc. of the AAAI
  Conference on Artificial Intelligence}, 2018, pp. 2669--2678.

\bibitem{Seto1998}
D.~Seto, B.~Krogh, L.~Sha, and A.~Chutinan, ``The simplex architecture for safe
  online control system upgrades,'' in \emph{Proc. of the American Control
  Conference}, 1998, pp. 3504--3508.

\bibitem{Phan2020}
D.~T. Phan, R.~Grosu, N.~Jansen, N.~Paoletti, S.~A. Smolka, and S.~D. Stoller,
  ``Neural simplex architecture,'' in \emph{Proc. of the NASA Formal Methods
  Symposium}, 2020, pp. 97--114.

\bibitem{Schurmann2021}
B.~Sch\"urmann, M.~Klischat, N.~Kochdumper, and M.~Althoff, ``Formal safety net
  control using backward reachability analysis,'' \emph{Transactions on
  Automatic Control}, vol.~67, no.~11, pp. 5698--5713, 2021.

\bibitem{Cheng2019}
R.~Cheng, G.~Orosz, R.~M. Murray, and J.~W. Burdick, ``End-to-end safe
  reinforcement learning through barrier functions for safety-critical
  continuous control tasks,'' in \emph{Proc. of the AAAI Conference on
  Artificial Intelligence}, 2019, pp. 3387--3395.

\bibitem{Marvi2021}
Z.~Marvi and B.~Kiumarsi, ``Safe reinforcement learning: {A} control barrier
  function optimization approach,'' \emph{International Journal of Robust and
  Nonlinear Control}, vol.~31, no.~6, pp. 1923--1940, 2021.

\bibitem{Bastani2021}
O.~Bastani, ``Safe reinforcement learning with nonlinear dynamics via model
  predictive shielding,'' in \emph{Proc. of the American Control Conference},
  2021, pp. 3488--3494.

\bibitem{Wabersich2021}
K.~P. Wabersich and M.~N. Zeilinger, ``A predictive safety filter for
  learning-based control of constrained nonlinear dynamical systems,''
  \emph{Automatica}, vol. 129, 2021, article 109597.

\bibitem{Shao2021}
Y.~S. Shao, C.~Chen, S.~Kousik, and R.~Vasudevan, ``Reachability-based
  trajectory safeguard ({RTS}): {A} safe and fast reinforcement learning safety
  layer for continuous control,'' \emph{Robotics and Automation Letters},
  vol.~6, no.~2, pp. 3663--3670, 2021.

\bibitem{Gillula2012}
J.~H. Gillula and C.~J. Tomlin, ``Guaranteed safe online learning via
  reachability: {T}racking a ground target using a quadrotor,'' in \emph{Proc.
  of the Int. Conference on Robotics and Automation}, 2012, pp. 2723--2730.

\bibitem{Mitchell2005}
I.~M. Mitchell, A.~M. Bayen, and C.~J. Tomlin, ``A time-dependent
  {Hamilton\textendash Jacobi} formulation of reachable sets for continuous
  dynamic games,'' \emph{Transactions on Automatic Control}, vol.~50, no.~7,
  pp. 947--957, 2005.

\bibitem{Selim2022}
M.~Selim, A.~Alanwar, S.~Kousik, G.~Gao, M.~Pavone, and K.~H. Johansson, ``Safe
  reinforcement learning using black-box reachability analysis,'' \emph{IEEE
  Robotics and Automation Letters}, vol.~7, no.~4, pp. 10\,665--10\,672, 2022.

\bibitem{Scott2016}
J.~K. Scott, D.~M. Raimondo, G.~R. Marseglia, and R.~D. Braatz, ``Constrained
  zonotopes: A new tool for set-based estimation and fault detection,''
  \emph{Automatica}, vol.~69, pp. 126--136, 2016.

\bibitem{Althoff2013a}
M.~Althoff, ``Reachability analysis of nonlinear systems using conservative
  polynomialization and non-convex sets,'' in \emph{Proc. of the Int.
  Conference on Hybrid Systems: Computation and Control}, 2013, pp. 173--182.

\bibitem{Kochdumper2020c}
N.~Kochdumper, B.~Sch\"urmann, and M.~Althoff, ``Utilizing dependencies to
  obtain subsets of reachable sets,'' in \emph{Proc. of the Int. Conference on
  Hybrid Systems: Computation and Control}, 2020, article 1.

\bibitem{Kochdumper2019}
N.~Kochdumper and M.~Althoff, ``Sparse polynomial zonotopes: A novel set
  representation for reachability analysis,'' \emph{Transactions on Automatic
  Control}, vol.~66, no.~9, pp. 4043--4058, 2021.

\bibitem{Althoff2008c}
M.~Althoff, O.~Stursberg, and M.~Buss, ``Reachability analysis of nonlinear
  systems with uncertain parameters using conservative linearization,'' in
  \emph{Proc. of the Int. Conference on Decision and Control}, 2008, pp.
  4042--4048.

\bibitem{Roehm2018a}
H.~Roehm, J.~Oehlerking, M.~Woehrle, and M.~Althoff, ``Model conformance for
  cyber-physical systems: A survey,'' \emph{Transactions on Cyber-Physical
  Systems}, vol.~3, no.~3, 2018, article 30.

\bibitem{Koschi2020}
M.~Koschi and M.~Althoff, ``Set-based prediction of traffic participants
  considering occlusions and traffic rules,'' \emph{Transactions on Intelligent
  Vehicles}, vol.~6, no.~2, pp. 249--265, 2020.

\bibitem{Bak2022}
S.~Bak, S.~Bogomolov, B.~Hencey, N.~Kochdumper, E.~Lew, and K.~Potomkin,
  ``Reachability of {K}oopman linearized systems using random {F}ourier feature
  observables and polynomial zonotope refinement,'' in \emph{Proc. of the Int.
  Conference on Computer Aided Verification}, 2022, pp. 490--510.

\bibitem{Grossmann2003}
I.~E. Grossmann and S.~Lee, ``Generalized convex disjunctive programming:
  {N}onlinear convex hull relaxation,'' \emph{Computational Optimization and
  Applications}, vol.~26, no.~1, pp. 83--100, 2003.

\bibitem{Chen2012}
X.~Chen, S.~Sankaranarayanan, and E.~{\'A}brah{\'a}m, ``Taylor model flowpipe
  construction for non-linear hybrid systems,'' in \emph{Proc. of the Real-Time
  Systems Symposium}, 2012, pp. 183--192.

\bibitem{Kochdumper2020b}
N.~Kochdumper and M.~Althoff, ``Reachability analysis for hybrid systems with
  nonlinear guard sets,'' in \emph{Proc. of the Int. Conference on Hybrid
  Systems: Computation and Control}, 2020, article 2.

\bibitem{Makino2003}
K.~Makino and M.~Berz, ``Taylor models and other validated functional inclusion
  methods,'' \emph{International Journal of Pure and Applied Mathematics},
  vol.~4, no.~4, pp. 379--456, 2003.

\bibitem{Dantzig2016}
G.~Dantzig, \emph{Linear Programming and Extensions}.\hskip 1em plus 0.5em
  minus 0.4em\relax Princeton University Press, 2016.

\bibitem{Balas1998}
E.~Balas, ``Disjunctive programming: Properties of the convex hull of feasible
  points,'' \emph{Discrete Applied Mathematics}, vol.~89, no.~1, pp. 3--44,
  1998.

\bibitem{Jaulin2006}
L.~Jaulin, M.~Kieffer, and O.~Didrit, \emph{Applied Interval Analysis}.\hskip
  1em plus 0.5em minus 0.4em\relax Springer Science \& Business Media, 2006.

\bibitem{Deits2015}
R.~Deits and R.~Tedrake, ``Computing large convex regions of obstacle-free
  space through semidefinite programming,'' in \emph{Proc. of the Int. Workshop
  on the Algorithmic Foundations of Robotics}, 2015, pp. 109--124.

\bibitem{Bak2015b}
S.~Bak, Z.~Huang, F.~A.~T. Abad, and M.~Caccamo, ``Safety and progress for
  distributed cyber-physical systems with unreliable communication,''
  \emph{Transactions on Embedded Computing Systems}, vol.~14, no.~4, 2015,
  article 76.

\bibitem{Althoff2015a}
M.~Althoff, ``An introduction to {CORA} 2015,'' in \emph{Proc. of the Int.
  Workshop on Applied Verification for Continuous and Hybrid Systems}, 2015,
  pp. 120--151.

\bibitem{Schulman2017}
J.~Schulman, F.~Wolski, P.~Dhariwal, A.~Radford, and O.~Klimov, ``Proximal
  policy optimization algorithms,'' \emph{arXiv preprint arXiv:1707.06347},
  2017.

\bibitem{Kelly2020}
M.~O'Kelly, H.~Zheng, D.~Karthik, and R.~Mangharam, ``F1tenth: {A}n open-source
  evaluation environment for continuous control and reinforcement learning,''
  \emph{Proceedings of Machine Learning Research}, vol. 123, pp. 77--89, 2020.

\bibitem{Kochdumper2021b}
N.~Kochdumper, F.~Gruber, B.~Sch{\"u}rmann, V.~Ga{\ss}mann, M.~Klischat, and
  M.~Althoff, ``{AROC}: {A} toolbox for automated reachset optimal controller
  synthesis,'' in \emph{Proc. of the Int. Conference on Hybrid Systems:
  Computation and Control}, 2021, article 23.

\bibitem{Thrun2001}
S.~Thrun, D.~Fox, W.~Burgard, and F.~Dellaert, ``Robust {M}onte {C}arlo
  localization for mobile robots,'' \emph{Artificial Intelligence}, vol. 128,
  no. 1-2, pp. 99--141, 2001.

\bibitem{Althoff2017a}
M.~Althoff, M.~Koschi, and S.~Manzinger, ``{CommonRoad}: Composable benchmarks
  for motion planning on roads,'' in \emph{Proc. of the IEEE Intelligent
  Vehicles Symposium}, 2017, pp. 719--726.

\bibitem{Koschi2017a}
M.~Koschi and M.~Althoff, ``{SPOT}: {A} tool for set-based prediction of
  traffic participants,'' in \emph{Proc. of the Intelligent Vehicles
  Symposium}, 2017, pp. 1686--1693.

\bibitem{Wang2021}
X.~Wang, H.~Krasowski, and M.~Althoff, ``{CommonRoad-RL}: {A} configurable
  reinforcement learning environment for motion planning of autonomous
  vehicles,'' in \emph{Proc. of the Int. Intelligent Transportation Systems
  Conference}, 2021, pp. 466--472.

\bibitem{Yuan2022}
Z.~Yuan, A.~W. Hall, S.~Zhou, L.~Brunke, M.~Greeff, J.~Panerati, and A.~P.
  Schoellig, ``{S}afe-{C}ontrol-{G}ym: {A} unified benchmark suite for safe
  learning-based control and reinforcement learning in robotics,'' \emph{IEEE
  Robotics and Automation Letters}, vol.~7, no.~4, pp. 11\,142--11\,149, 2022.

\bibitem{Gassmann2020}
V.~Ga{\ss}mann and M.~Althoff, ``Scalable zonotope-ellipsoid conversions using
  the {E}uclidean zonotope norm,'' in \emph{Proc. of the American Control
  Conference}, 2020, pp. 4715--4721.

\bibitem{Pek2018b}
C.~Pek and M.~Althoff, ``Computationally efficient fail-safe trajectory
  planning for self-driving vehicles using convex optimization,'' in
  \emph{Proc. of the Int. Conference on Intelligent Transportation Systems},
  2018, pp. 1447--1454.

\bibitem{Kochdumper2022}
N.~Kochdumper, ``Extensions of polynomial zonotopes and their application to
  verification of cyber-physical systems,'' Ph.D. dissertation, Technical
  University of Munich, 2022.

\bibitem{Papadimitriou1981}
C.~H. Papadimitriou, ``On the complexity of integer programming,''
  \emph{Journal of the ACM}, vol.~28, no.~4, pp. 765--768, 1981.

\bibitem{Schuermann2018a}
B.~Sch{\"u}rmann, N.~Kochdumper, and M.~Althoff, ``Reachset model predictive
  control for disturbed nonlinear systems,'' in \emph{Proc. of the Int.
  Conference on Decision and Control}, 2018, pp. 3463--3470.

\bibitem{Wetzlinger2021}
M.~Wetzlinger, A.~Kulmburg, and M.~Althoff, ``Adaptive parameter tuning for
  reachability analysis of nonlinear systems,'' in \emph{Proc. of the Int.
  Conference on Hybrid Systems: Computation and Control}, 2021, article 16.

\bibitem{Wetzlinger2022}
M.~Wetzlinger, N.~Kochdumper, S.~Bak, and M.~Althoff, ``Fully-automated
  verification of linear systems using inner-and outer-approximations of
  reachable sets,'' \emph{arXiv preprint arXiv:2209.09321}, 2022.

\end{thebibliography}

\end{document}